\documentclass[journal]{IEEEtran}
\ifCLASSINFOpdf
   \usepackage[pdftex]{graphicx}
\else
\fi
\usepackage{graphicx}
\usepackage{subcaption}
\usepackage{amsmath}
\usepackage[version=4]{mhchem}
\usepackage{siunitx}
\usepackage{longtable,tabularx}
\usepackage{algorithm}
\usepackage{algpseudocode}
\usepackage{amsthm}
\usepackage{amssymb}
\usepackage{multirow}
\usepackage{caption,psfrag}
\usepackage{amsfonts}
\usepackage{xcolor}
\usepackage{mathtools}
\newtheorem{prop}{\textit{Proposition}}
\newtheorem{obs}{\textit{Observation}}
\newtheorem{theorem}{\textit{Theorem}}
\theoremstyle{definition}
\newtheorem{definition}{Definition}
\newtheorem{lemma}{\textit{Lemma}}

\newcommand{\argmin}{\operatorname{argmin}}
\newcommand{\argmax}{\operatorname{argmax}}
\DeclarePairedDelimiter\Floor\lfloor\rfloor

\title{ Mathematical Properties of Generalized Shape Expansion-Based Motion Planning Algorithms}
\author{Adhvaith Ramkumar$^{1}$, Vrushabh Zinage$^{1}$ and Satadal Ghosh$^{1}$
\thanks{$^{1}$Adhvaith Ramkumar, Vrushabh Zinage and Satadal Ghosh are with the Department of Aerospace Engineering, Indian Institute of Technology Madras, India
        {\tt\small $\{$ae16b018,ae16b017,satadal$\}$@smail.iitm.ac.in}}
}
\begin{document}
\bibliographystyle{IEEEtran}

\maketitle

\begin{abstract}
Motion planning is an essential aspect of autonomous systems and robotics and is an active area of research. A recently-proposed sampling-based motion planning algorithm, termed 'Generalized Shape Expansion' (GSE), has been shown to possess significant improvement in computational time over several existing well-established algorithms. The GSE has also been shown to be probabilistically complete. However, asymptotic optimality of the GSE is yet to be studied. To this end, in this paper we show that the GSE algorithm is not asymptotically optimal by studying its behaviour for the promenade problem. In order to obtain a probabilistically complete and asymptotically optimal generalized shape-based algorithm, a modified version of the GSE, namely '$\text{GSE}^{\star}$' algorithm, is subsequently presented. The forementioned desired mathematical properties of the  $\text{GSE}^{\star}$ algorithm are justified by its detailed analysis. Numerical simulations are found to be in line with the theoretical results on the $\text{GSE}^{\star}$ algorithm. 
\end{abstract}


\IEEEpeerreviewmaketitle

\section{Introduction}
Unmanned vehicle Systems (UxS) have seen a steady increase in attention over the past decade due to rapid advancements in technology and reduced human risk associated to it. Applications of UxS range from civilian ones like pollutant monitoring to defense-related ones like surveillance and reconnaissance. Motion planning being a key aspect of any UxS is an active area of research.
 
An early seminal paper in the field of motion planning utilized artificial potential fields \cite{Khatib_1}, however it was shown to suffer from the problem of local minima \cite{Koren_1}. Numerous online motion planning strategies have been formulated using several methods like velocity obstacles \cite{vo_1} and its extensions such as optimal reciprocal collision avoidance \cite{orca_1}, collision cone approach \cite{ chakravarthy2011collision, chakravarthy2012generalization} 
gradient vector fields \cite{davidcasbeer}, 
pseudospectral methods \cite{jgcd_gong2009pseudospectral}
and terminal angle-constrained guidance theory \cite{Ghosh_1}, \cite{jgcd_ghosh2017unmanned}. In the presence of prior knowledge about the environment, offline motion planners such as discrete search-based  methods \cite{Brooks1985ASA,Mitchell1992} and sampling-based methods  \cite{prm,sertac_karaman2011sampling_rrt_star} are employed. Optimization-based  techniques such as sequential convex programming \cite{Chen,Augugliaro} and quadratic programming methods \cite{upenn} are also used. 

Among offline planners, sampling-based algorithms are a  highly favoured approach due to computational advantages in higher-dimensional spaces. A sampling-based algorithm generates a connected graph, containing the initial and goal points, entirely inside the free-space. This is accomplished by obtaining random samples from the free-space and successively adding them to the graph, eventually connecting the initial and goal points. The optimal path is then found among the paths between the initial and goal points. Several sampling-based algorithms have been developed in literature such as  Probabilistic Road Maps (PRM) \cite{prm}, Rapidly-exploring Random Trees (RRT)\cite{rrt_attraction_sequence} and Fast Marching Tree (FMT$^\star$) \cite{janson2015fast}. A spherical expansion-based motion planning algorithm, termed 'SE-SCP' \cite{sescp_baldini2016fast} was shown to be computationally more efficient than PRM$^\star$ and RRT$^\star$  via numerical studies. In order to better leverage obstacle-space information, an algorithm termed 'Generalized Shape Expansion' (GSE) algorithm, relying on a more generic shape expansion has recently been presented in \cite{gse} for motion planning in 2-D obstacle-cluttered environments. 
Its extension to 3-D environment was also presented in \cite{3d_gse_journal,3d_gse}. The GSE algorithm leveraged a novel generalized shape for the expansion of the generated graph, which helped in exploring the free-space in a fast yet more extensive manner. 
Through extensive numerical simulation studies, the performance of the GSE in terms of computation time was found to be significantly better than several other well-established sampling-based methods. An even faster planner has been presented in \cite{lcss_2020directional_gse} by embedding directional sampling scheme to the GSE for 2-D environments.   
While sampling-based motion planning algorithms cannot provide deterministic guarantees, two probabilistic criteria are used to evaluate the utility of these algorithms. These properties are probabilistic completeness and asymptotic optimality. Probabilistic completeness of an algorithm guarantees that the probability of failure of an algorithm decays to zero, as the number of samples goes to infinity, while asymptotic optimality of an algorithm guarantees that the algorithm is highly likely to find the optimal path as the number of samples goes to infinity. Algorithms like PRM and RRT were found to be probabilistically complete \cite{prmprobcomp,rrt_attraction_sequence}, but they are not asymptotically optimal \cite{sertac_karaman2011sampling_rrt_star}. In order to ensure asymptotic optimality, their modified versions such as RRT$^\star$ and PRM$^\star$ were developed and analyzed in \cite{sertac_karaman2011sampling_rrt_star}.

{This paper is evolved from \cite{probabilistic_completeness_gse}, where the  probabilistic completeness property of the GSE algorithm was established. }
However asymptotic optimality of the GSE algorithm has not yet been studied so far. To this end, in this paper, we consider the promenade problem, for which following the methodology of \cite{nechushtan} we find that the GSE has a non-zero probability of returning low-quality solutions for the promenade problem. This establishes the lack of asymptotic optimality of the GSE algorithm. Subsequently, we present a modification of the GSE algorithm, termed as the GSE$^\star$ algorithm, in order to ensure asymptotic optimality as well as probabilistic completeness. Besides the GSE expansion, the $\text{GSE}^{\star}$ algorithm adds additional vertices and edges based on a decreasing connection radii criterion to the graph generated by the GSE, thus limiting the path cost while further exploring the environment. 
The basic GSE graph, a subset of the $\text{GSE}^{\star}$ graph, also helps in retaining computational advantage. A detailed analysis is carried out of the GSE$^\star$ algorithm for theoretical verification of both probabilistic completeness and asymptotic optimality.

The paper is organized as follows. The problem description along with a brief description of the probabilistic completeness and asymptotic optimality of a path planning problem is given in Section \ref{sec:Problem}. We provide the $d$-dimensional GSE algorithm in Section \ref{sec:GSE}. This is followed by the study of mathematical properties of the GSE algorithm 
in Section \ref{sec:probcomp} and \ref{GSEasopt}. Next, we present the $\text{GSE}^{\star}$ algorithm in Section \ref{GSE*alg} followed by an analysis of the same in Section \ref{GSE*probcomp} and \ref{GSE*asopt}. Simulation studies are presented to compare GSE and $\text{GSE}^{\star}$ 
in Section \ref{sec:results}, followed by the conclusion in Section \ref{sec:conclusion}.

\begin{figure*}[]
\captionsetup[subfigure]{justification=centering}
\centering
\begin{subfigure}{0.23\textwidth}
\fbox{\includegraphics[scale=0.12]{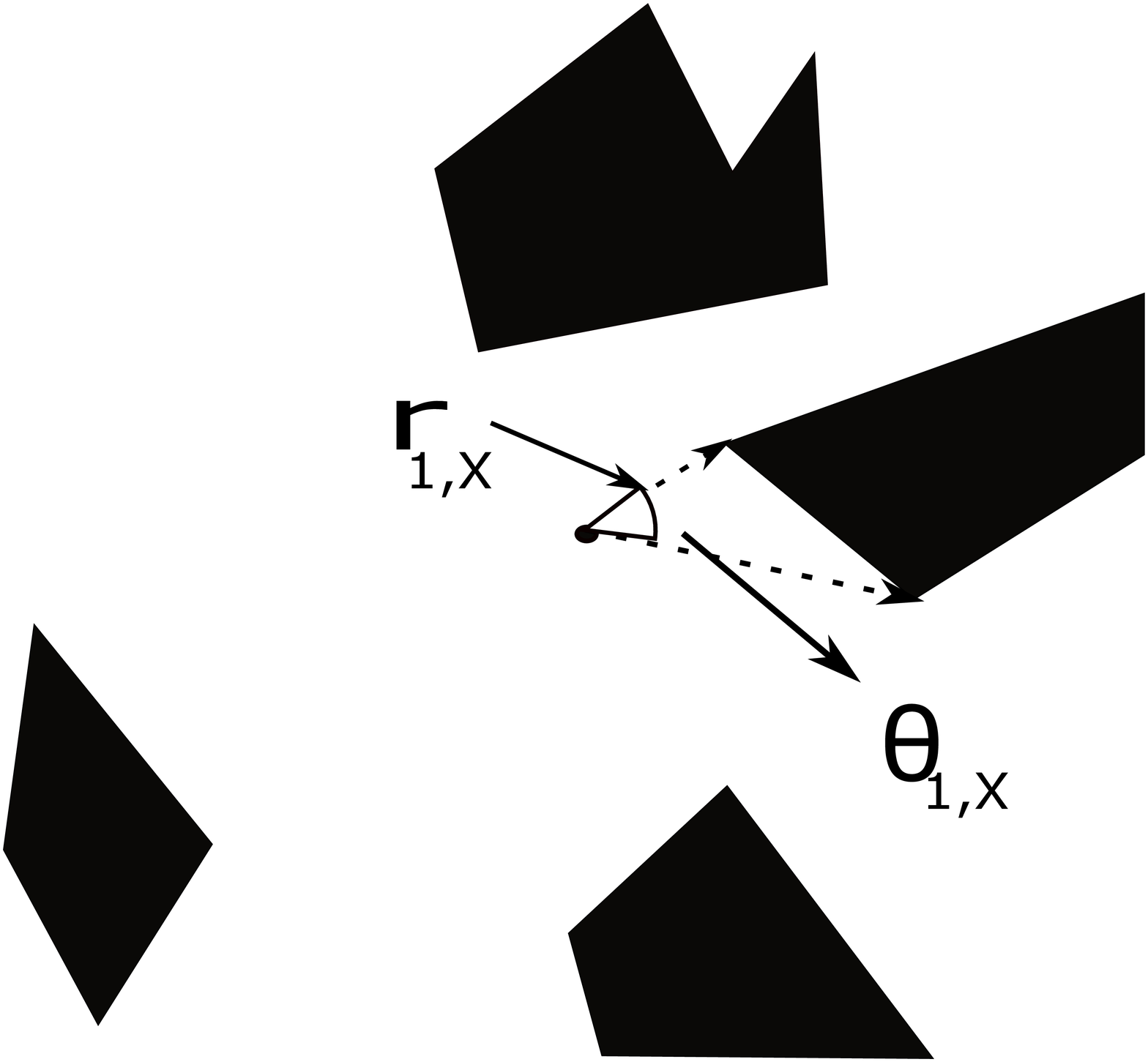}}
\caption{$\boldsymbol{r}_{1,X}$ and ${\theta_{1,X}}$ }
\label{}
\end{subfigure}
\begin{subfigure}{0.23\textwidth}
\fbox{\includegraphics[scale=0.12]{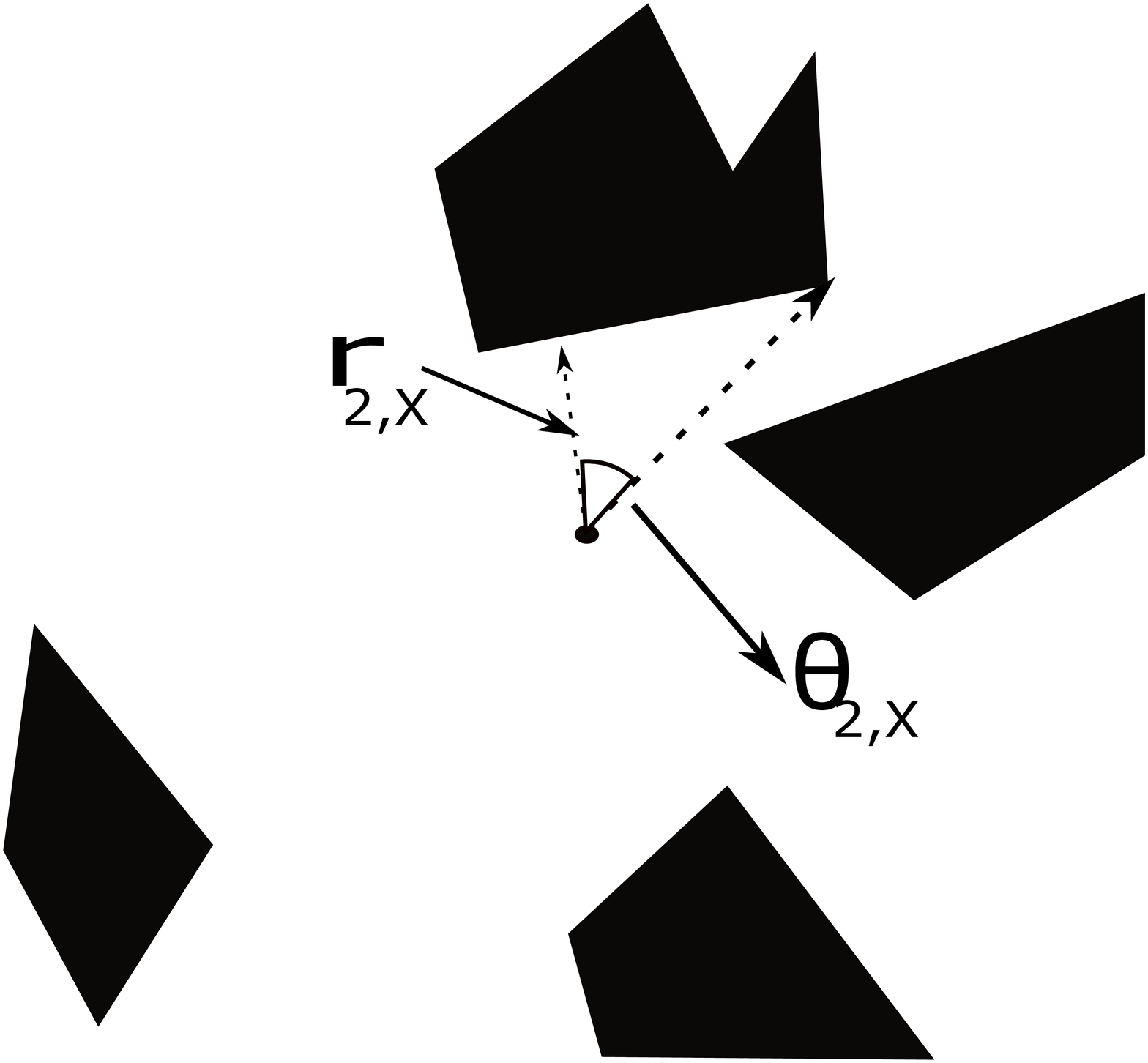}}
\caption{$\boldsymbol{r}_{2,X}$ and ${\theta_{2,X}}$}
\label{}
\end{subfigure}
\begin{subfigure}{0.23\textwidth}
\fbox{\includegraphics[scale=0.12]{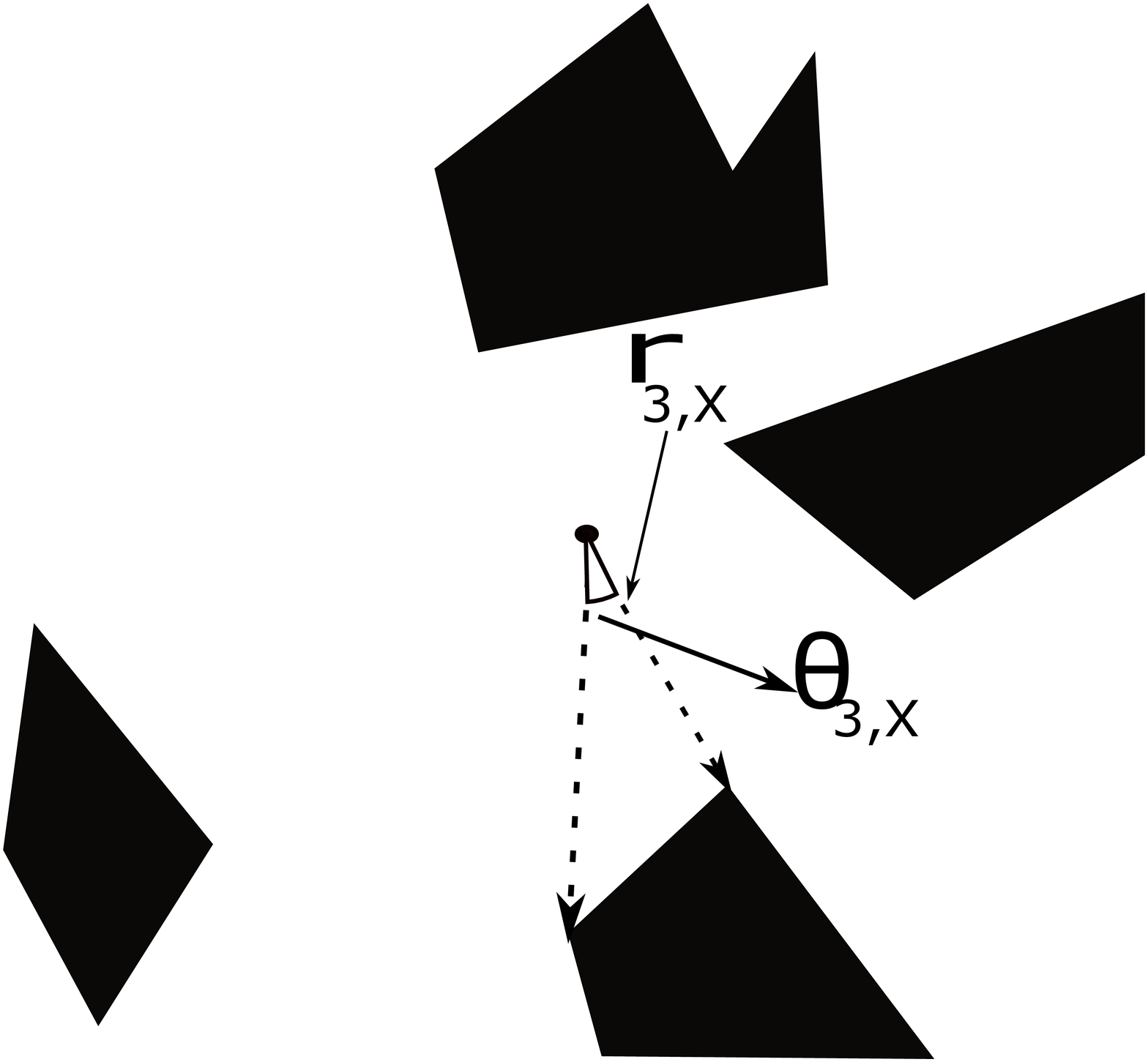}}
\caption{$\boldsymbol{r}_{3,X}$ and ${\theta_{3,X}}$}
\label{}
\end{subfigure}
\begin{subfigure}{0.23\textwidth}
\fbox{\includegraphics[scale=0.12]{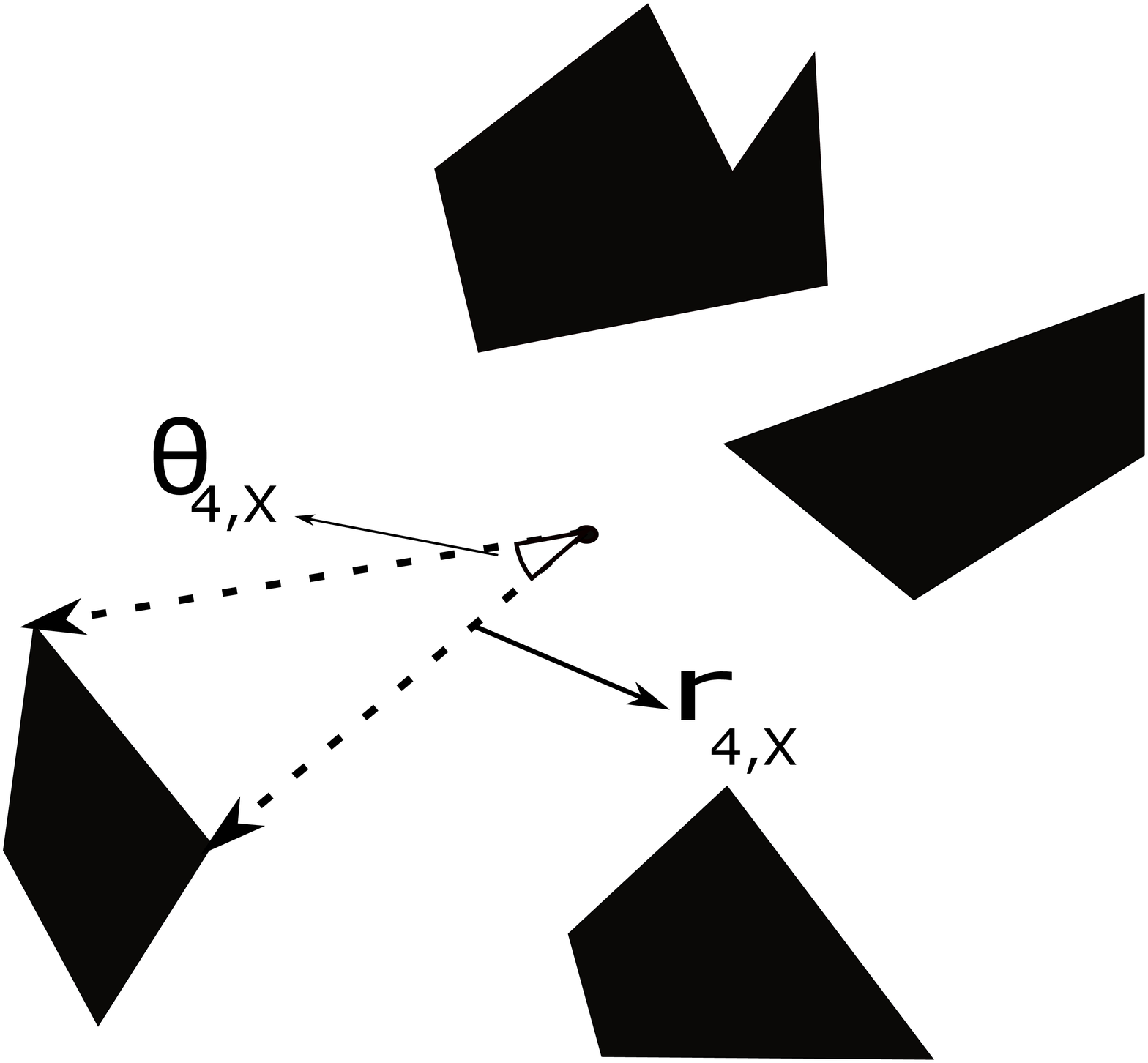}}
\caption{$\boldsymbol{r}_{4,X}$ and ${\theta_{4,X}}$}
\label{}
\end{subfigure}
\begin{center}
\captionsetup[subfigure]{justification=centering}
\centering

\end{center}
\caption{Depiction of $\theta_{i,X}$ and $\boldsymbol{r}_{i,X}$ for all $i\in\{1,2,3,4\}$ for a 2-D environment}
\label{fig:generalized_shape_d}
\end{figure*}
\section{Problem Description} \label{sec:Problem}
In this section, we introduce the general path planning problem along with the probabilistic completeness and asymptotic optimality of the path planning problem. Finally the problem statement is introduced which we address in this paper. 
\subsection{Path Planning Problem}
Consider a compact configuration space $\mathbb{X}=[0,1]^d$, where $d$ is the dimension of the workspace. Let the obstacles in the configuration space be denoted as $\mathbb{X}_\text{obs}$. The free-space in $\mathbb{X}$ is then given by $\mathbb{X}_{\text{free}}= \mathbb{X}\setminus\mathbb{X}_\text{obs}$. 
  We assume that the free-space is open and the obstacle space is closed. The Start point $\boldsymbol{X}_{\text {init }}$  is an element of $\mathbb{X}_{\text {free }}$, and the Goal point $\boldsymbol{X}_{\text {goal }}\in\mathbb{X}_\text{free}$. 
  Given a function $\sigma:[0,1] \rightarrow \mathbb{X}$,  the total variation of $\sigma$ is defined as follows,
  \begin{align}
   \operatorname{TV}(\sigma)=\sup _{\left\{n \in \mathbb{N}, 0=b_{0}<\cdots<b_{n}=1\right\}} \sum_{i=1}^{n}\left|\sigma\left(b_{i}\right)-\sigma\left(b_{i-1}\right)\right|   
  \end{align}
Where, the supremum is taken over all partitions of the set $[0,1]$. The symbol $\|.\|$ denotes the Euclidean norm.
A function $\sigma$ with $\mathrm{TV}(\sigma)<\infty$ is said to have bounded variation.

\begin{definition}
\textbf{(Path)}: A function $\sigma:[0,1] \rightarrow \mathbb{X}$ of bounded variation is called a
\begin{itemize}
    \item Path, if it is continuous (Bounded variation of $\sigma$ is equivalent to a path of finite length)
    \item Collision-free path, if it is a path, and $\sigma(b) \in \mathbb{X}_{\text {free }},$ for all $b\in[0,1]$
    \item Feasible path, if it is a collision-free path, $\sigma(0)=\boldsymbol{X}_{\text {init }},$ and $\sigma(1)=\boldsymbol{X}_{\text{goal}}$ 
\end{itemize}
\end{definition}

Let $\mathfrak{S}$ denote the set of all collision-free paths in the given configuration space. We define the norm $\|{\sigma}\|_{\text{BV}}=\mathrm{TV}(\sigma)+\int_{0}^{1}|\sigma(t)|dt$. The topology on the space $\mathfrak{S}$ is derived from the metric induced by this norm. Define a cost function $c:\mathfrak{S} \rightarrow \mathbb{R}_{\geq0}$. We assume that the cost function increases with the length of path, that is $c(\sigma_1) > c(\sigma_2)$ if $\mathrm{TV}(\sigma_1) > \mathrm{TV}(\sigma_2)$. 
Further note that a path $\sigma:[0,1]\rightarrow \mathbb{X}$ is said to have 
 \begin{itemize}
     \item strong $\delta$-clearance if for each $s \in [0,1]$,  $\mathcal{B}_{\delta}(\sigma(s))\subseteq \mathbb{X}_{\text{free}}$ where $\mathcal{B}_{r}(P)=\{x \mid \|x-P\|< r\}$. 
     \item  weak $\delta$-clearance if there exists a homotopy, i.e. a continuous function $\psi:[0,1]\rightarrow \mathfrak{S}$ with $\psi(0)=\sigma$ and $\psi(1)=\sigma_{1}$ where $\sigma_1$ has strong $\delta$-clearance. Further if $\alpha \in [0,1]$ then the path $\psi(\alpha)$ has strong $\delta_{\alpha}$-clearance for some $\delta_{\alpha}>0$.
 \end{itemize}
 
We also assume the following:
\begin{itemize}
    \item There exists a feasible path $\sigma$ between the start and goal points with strong $\epsilon$-clearance for some $\epsilon >0$.
    \item There exists a path $\sigma^\star$ with weak $\delta$-clearance for some $\delta>0$ and optimal cost $c^\star$, such that if there exists a sequence $(\sigma_n)_{n=1}^{\infty}$ with $\underset{n\to\infty}{\lim} \sigma_n = \sigma^\star$, then $ \underset{n\to\infty}{\lim} c(\sigma_n) = c^\star$.
\end{itemize}

A path planning problem entails finding a feasible path in an environment characterized by the triplet $\left(\mathbb{X}_{\text {free}}, \boldsymbol{X}_{\text {init}}, \boldsymbol{X}_{\text {goal }}\right)$. 
  
\subsection{Probabilistic Completeness}
Sampling-based motion planning algorithms  rely on uniform random sampling from the free-space in each step. Consequently, it is difficult to provide deterministic guarantees on the success of the algorithm. We can however study the claim that success in finding a feasible path from $\boldsymbol{X}_\text{init}$ to $\boldsymbol{X}_\text{goal}$ is guaranteed as the number of iterations of the algorithm increases to infinity. This constitutes the fundamental idea of probabilistic completeness of a algorithm.

We define a few preliminary quantities

\begin{itemize}
    \item  We denote the graph generated by an algorithm after $n$ steps  ${G}_n=(\mathbb{V}_n,\mathbb{E}_n)$, where $\mathbb{V}_n$ and $\mathbb{E}_n$ are the vertex and edge sets, respectively.
    \item Let $\mathcal{P}_n$ be an indicator random variable that takes value one indicating the event that there is a connected path between $\boldsymbol{X}_{\text{init}}$ and $\boldsymbol{X}_{\text{goal}}$ in ${G}_n$.
\end{itemize}
The  algorithm can be said to possess the probabilistic completeness property if 
\begin{align}
     \underset{n\to\infty}{\limsup}\;\mathbb{P}(\mathcal{P}_n=0)=0 
     \label{eq:prob_compldefn}
\end{align}
where, $\mathbb{P}$ denotes probability of an event.
\subsection{\label{defasyopt}Asymptotic Optimality}
Let $Y_n$ denote the cost of the least cost solution returned by a sampling-based motion planning algorithm in $n$ iterations. We define $c^\star=\text{inf} \{c(\sigma) : \text{$\sigma$ is a feasible path}\}$. 
An algorithm is said to be asymptotically optimal if 
\begin{equation}
    \mathbb{P}(\{\underset{n\to\infty}{\limsup}\;{Y_n}= c^{\star}\})=1
\end{equation}
\subsection{Problem Statement} \label{subsec:problem statement}
In this paper, first the GSE algorithm, which was proposed for 2-D and 3-D environments in \cite{gse,3d_gse_journal} respectively, is extended to a general $d$ dimensional configuration space. Then, probabilistic completeness of the GSE algorithm is to be studied first followed by an analysis on cost of paths returned by the GSE algorithm.

If the GSE algorithm is not found to be asymptotically optimal, then an amended version of the GSE that would satisfy both the probabilistic completeness and asymptotic optimality is to be presented and analyzed.

\section{Generalized Shape Expansion (GSE) Algorithm} \label{sec:GSE description}
  Let there be $m$ obstacles in configuration space $\mathbb{X}$. We assume that the points on the obstacle boundary in $\mathbb{X}$ are known.
Let the set of points on the $i^{\text{th}}$ obstacle be denoted by $\mathbb{Q}_i$, where $i\in\{1,\dots m\}$. Consider a point $\boldsymbol{X} \in \mathbb{X} \subset \mathbb{R}^{d}$. The minimum distance vector of  $\boldsymbol{X}$ from the $i^{\text{th}}$ obstacle is denoted by $\boldsymbol{r}_{i,X}$, and its magnitude is given by $\|\boldsymbol{r}_{i,X}\|$. Without loss of generality, obstacles are numbered such that $\|\boldsymbol{r}_{i,X}\|\leq \|\boldsymbol{r}_{j,X}\|$ for $i<j$. Let $\theta_{i,X}$ denotes the maximum angle made by the vector $\boldsymbol{r}_{i,X}$ with $(\boldsymbol{X}-\boldsymbol{x})$, where $\boldsymbol{x}\in \mathbb{Q}_i$. In other words,
\begin{align}
  \theta_{i,X}= \underset{\boldsymbol{x}\in \mathbb{Q}_{{i}}}\argmax\; \text{cos}^{-1}\left(\frac{(\boldsymbol{X}-\boldsymbol{x}).\boldsymbol{r}_{i,X}}{\|\boldsymbol{X}-\boldsymbol{x}\|\;\|\boldsymbol{r}_{i,X}\|}\right) \; i\in\{1,\dots,m\}
  \label{eqn:theta}
\end{align}
where $\boldsymbol{a}.\boldsymbol{b}$ denotes the dot product of vectors $\boldsymbol{a}$ and $\boldsymbol{b}$. Note that the formulation of the maximum angular spectrum covering $i^{\text{th}}$ obstacle as seen from a point $\boldsymbol{X}$ given above in Eq. \eqref{eqn:theta} is a generalization of the formulation presented in \cite{3d_gse_journal}. A sample depiction of $\theta_{i,X}$ and $\boldsymbol{r}_{i,X}$ for four obstacles in two-dimensional configuration space is shown in Fig. \ref{fig:generalized_shape_d}.
\subsection{\label{sec:GSE}Description of the GSE Algorithm}
\begin{algorithm}[]
\caption{GSE Algorithm}
\begin{algorithmic}[1]
\State ${\mathbb{V}\gets \{\boldsymbol{X_{\text{init}}}, \boldsymbol{X_{\text{goal}}}\}},\;\;\mathbb{E}\gets \emptyset$
\State ${\mathcal{S}_{\text{init}}(\boldsymbol{P})\gets \texttt{Shape}(\boldsymbol{P},\boldsymbol{X_{\text{init}}},\mathbb{X}_{\text{obs}})}$
\State ${\mathcal{S}_{\text{goal}}(\boldsymbol{P})\gets \texttt{Shape}(\boldsymbol{P},\boldsymbol{X_{\text{goal}}},\mathbb{X}_{\text{obs}})}$
\For{$j=1\dots n$}
\State $\boldsymbol{X_{\text{rand}}}\gets \texttt{GenerateSample}$
\State $\mathcal{S}_{\text{rand}}(\boldsymbol{P})\gets \texttt{Shape}(\boldsymbol{P},\boldsymbol{X_{\text{rand}}},\mathbb{X}_{\text{obs}})$
\State $\boldsymbol{X_{\text{nearest}}}\gets \texttt{Nearest}(\mathbb{G}=(\mathbb{V},\mathbb{E}),\boldsymbol{X_{\text{rand}}})$
\State $\mathcal{S}_{\text{nearest}}(\boldsymbol{P})\gets \texttt{Shape}(\boldsymbol{P},\boldsymbol{X_{\text{nearest}}},\mathbb{X}_{\text{obs}})$
\State $\boldsymbol{X_{\text{new}}}\gets \texttt{Steer-GSE}(\boldsymbol{X_{\text{nearest}}},\boldsymbol{X_{\text{rand}}})$
\State $\mathcal{S}_{\text{new}}(\boldsymbol{P})\gets \texttt{GSE-Shape}(\boldsymbol{P},\boldsymbol{X_{\text{new}}},\mathbb{X}_{\text{obs}})$
\State $\mathbb{X_{\text{intersect}}} \gets \texttt{IntersectShape}(G,\boldsymbol{X_{\text{new}}})$
\State $\mathbb{V} \gets \mathbb{V} \cup \{\boldsymbol{X_{\text{new}}},\boldsymbol{X_{\text{rand}}}\}$
\For{$\boldsymbol{X_{\text{n}}}\in \mathbb{X}_{\text{intersect}}$}
\State $\mathbb{E}\gets \mathbb{E} \cup \{(\boldsymbol{X_{\text{n}}},\boldsymbol{X_{\text{new}}})\}$
\EndFor
\State $\mathbb{G}=(\mathbb{V},\mathbb{E})$
\State $Path_{\text{init,goal,shortest}} \gets\texttt{MinPath}(G,\boldsymbol{X_{\text{init}}},\boldsymbol{X}_{\text{goal}})$
\EndFor
\end{algorithmic}
\label{algo:GSE_algo}
\end{algorithm}
\subsubsection{Initialization}

During initialization, the start and goal positions $\boldsymbol{X}_\text{init}$ and $\boldsymbol{X}_\text{goal}$, respectively, are added to the vertex set $\mathbb{V}$ with their corresponding parameters $\{\boldsymbol{r}_{{i,X}},\theta_{{i,X}}\}^m_{i=1}$. Values of these parameters are then used to compute the corresponding shape functions $\mathcal{S}_{\text{init}}(\boldsymbol{P})$ and $\mathcal{S}_{\text{goal}}(\boldsymbol{P})$, respectively, using the function $\texttt{Shape}(\boldsymbol{P},\boldsymbol{X},\mathbb{X}_{\text{obs}})$ for $\boldsymbol{X}=\boldsymbol{X_\text{init}}$ and $\boldsymbol{X_\text{goal}}$, respectively (Lines 2, 3 of Algorithm \ref{algo:GSE_algo}). The edge set $\mathbb{E}$ is set empty. This is shown in Lines 1-3 of Algorithm \ref{algo:GSE_algo}.

\subsubsection{Sample point} 

A random point $\boldsymbol{X}_\text{rand}\in \mathbb{X}_\text{free}$, is drawn from a uniform distribution such that the sequence of sampled points is independent and identically distributed
as shown in Line 5 of Algorithm \ref{algo:GSE_algo}.
\subsubsection{Nearest point} 

For a given $\boldsymbol{X}_\text{rand}$, the next step is to find a point $\boldsymbol{X}_\text{nearest}\in\mathbb{V}$, which is nearest to $\boldsymbol{X}_\text{rand}$ (see Line 7 of Algorithm \ref{algo:GSE_algo}).
\begin{equation}
\texttt{Nearest}(\mathbb{G}=(\mathbb{V},\mathbb{E}),\boldsymbol{X}_{\text{rand}} ) = \underset{\boldsymbol{X}\in \mathbb{V}}\argmin\|\boldsymbol{X}-\boldsymbol{X}_{\text{rand}}\|
\end{equation}

\begin{algorithm}[]
\caption{$\text{GSE Shape Function} \;(\texttt{Shape}(\boldsymbol{P},\boldsymbol{X},\mathbb{X}_{\text{obs}}))$}
\label{alg:shape_function}
\begin{algorithmic}[1]
\State ${\ell_0=0}$
\For{$i=1,2,\dots,m$}
\State $\ell_i\gets\texttt{sat}{(\text{cos}^{-1}(\frac{\boldsymbol{r_{i,X}}.{(\boldsymbol{P}-\boldsymbol{X})}}{\|\boldsymbol{r_{i,X}}\|\;\|{\boldsymbol{P}-\boldsymbol{X}}\|})-\theta_{i,X})}$
\State $f_i\gets\texttt{sat}(\theta_{i,X}-\texttt{cos}^{-1}(\frac{\boldsymbol{r_{i,X}}.{(\boldsymbol{P}-\boldsymbol{X})}}{\|\boldsymbol{r_{i,X}}\|\;\|{\boldsymbol{P}-\boldsymbol{X}}\|}))+\texttt{sat}(r_{i,X}-\|\boldsymbol{P}-\boldsymbol{X}\|)$
\State $g_{{i}}\gets f_{{i}}+{\sum_{j=1}^{i-1} \ell_j} -(i+1)$
\EndFor
\If{$\texttt{all}(\ell_{{i}})=1\;$} 
\State $\texttt{Shape}(\boldsymbol{P},\boldsymbol{X},\mathbb{X}_{\text{obs}})=0$
\Else
\State$\mathcal{S}_{\boldsymbol{X}}(\boldsymbol{P})=\texttt{Shape}(\boldsymbol{P},\boldsymbol{X},\mathbb{X}_{\text{obs}})\gets g_1*g_2*\dots g_m$
\EndIf
\end{algorithmic}
\end{algorithm}
\begin{figure*}[]
\captionsetup[subfigure]{justification=centering}
\centering
\begin{subfigure}{0.23\textwidth}
{\includegraphics[scale=0.08]{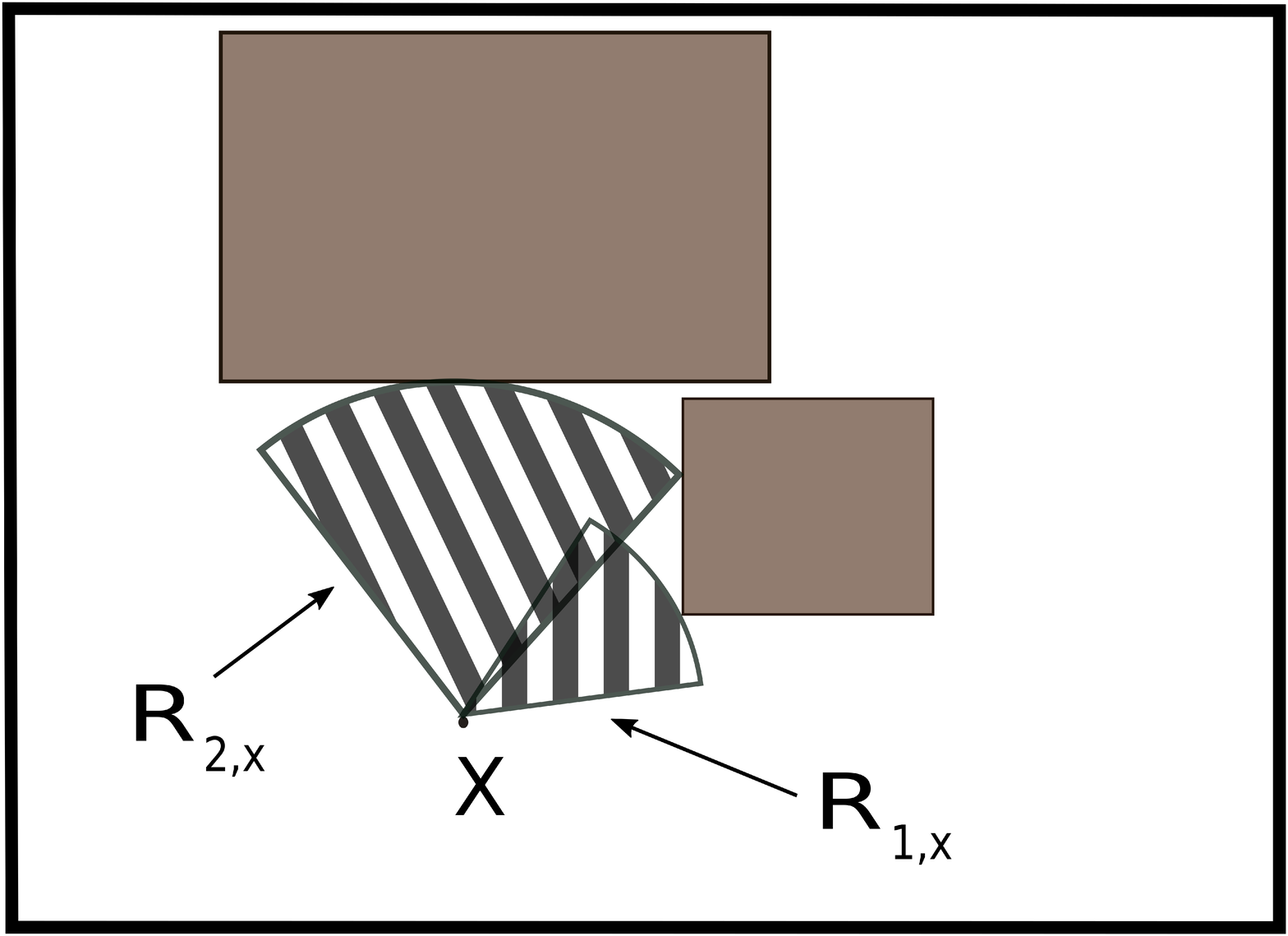}}
\caption{${\mathcal{R}_{1,\boldsymbol{X}}}$ and ${\mathcal{R}_{2,\boldsymbol{X}}}$ }
\label{fig:r1_and_r2}
\end{subfigure}
\begin{subfigure}{0.23\textwidth}
{\includegraphics[scale=0.08]{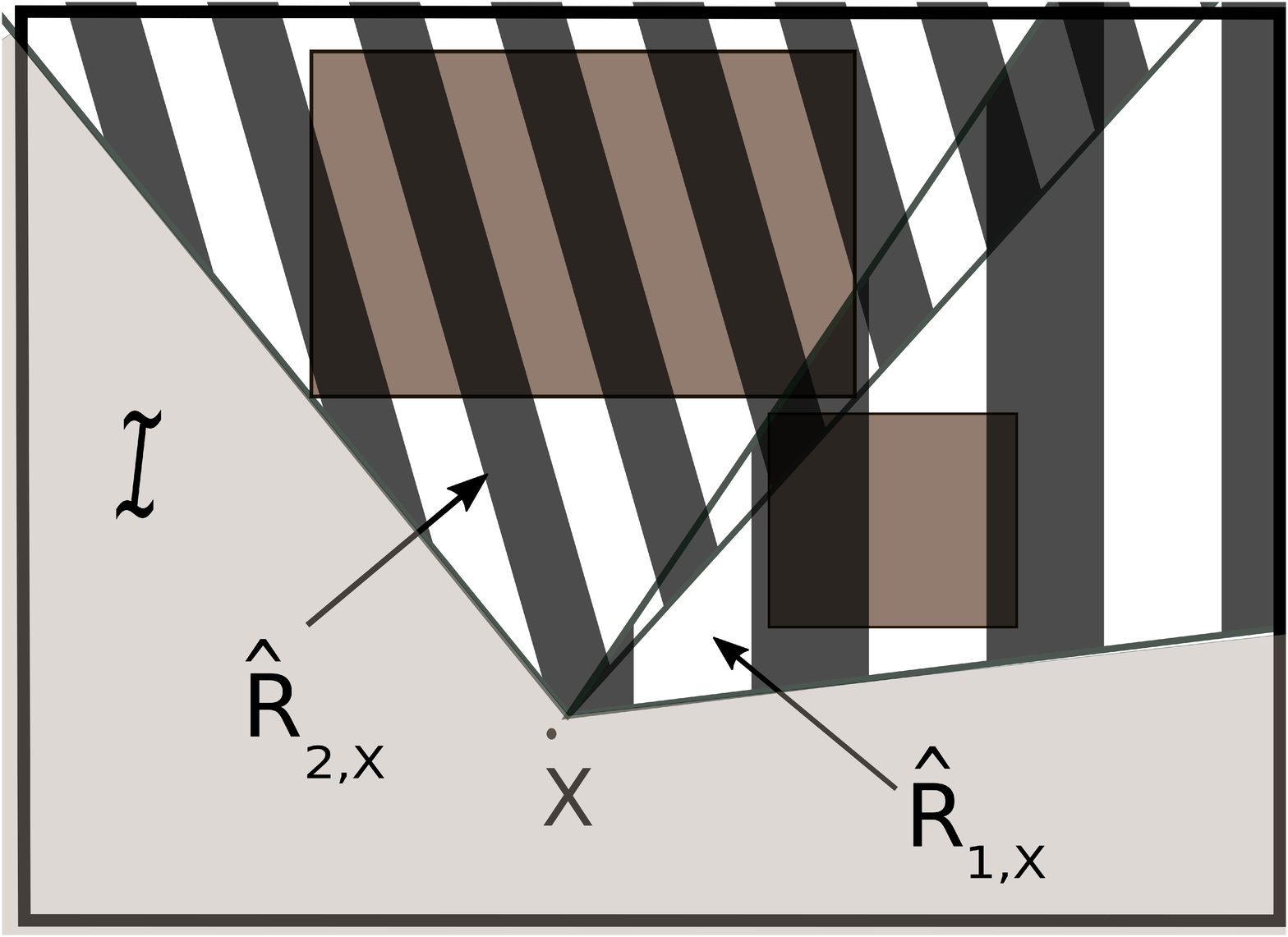}}
\caption{$\widehat{\mathcal{R}_{1,\boldsymbol{X}}}$, $\widehat{\mathcal{R}_{2,\boldsymbol{X}}}$ and $\mathcal{I}$}
\label{fig:r1_and_r2_hats}
\end{subfigure}
\begin{subfigure}{0.23\textwidth}
{\includegraphics[scale=0.08]{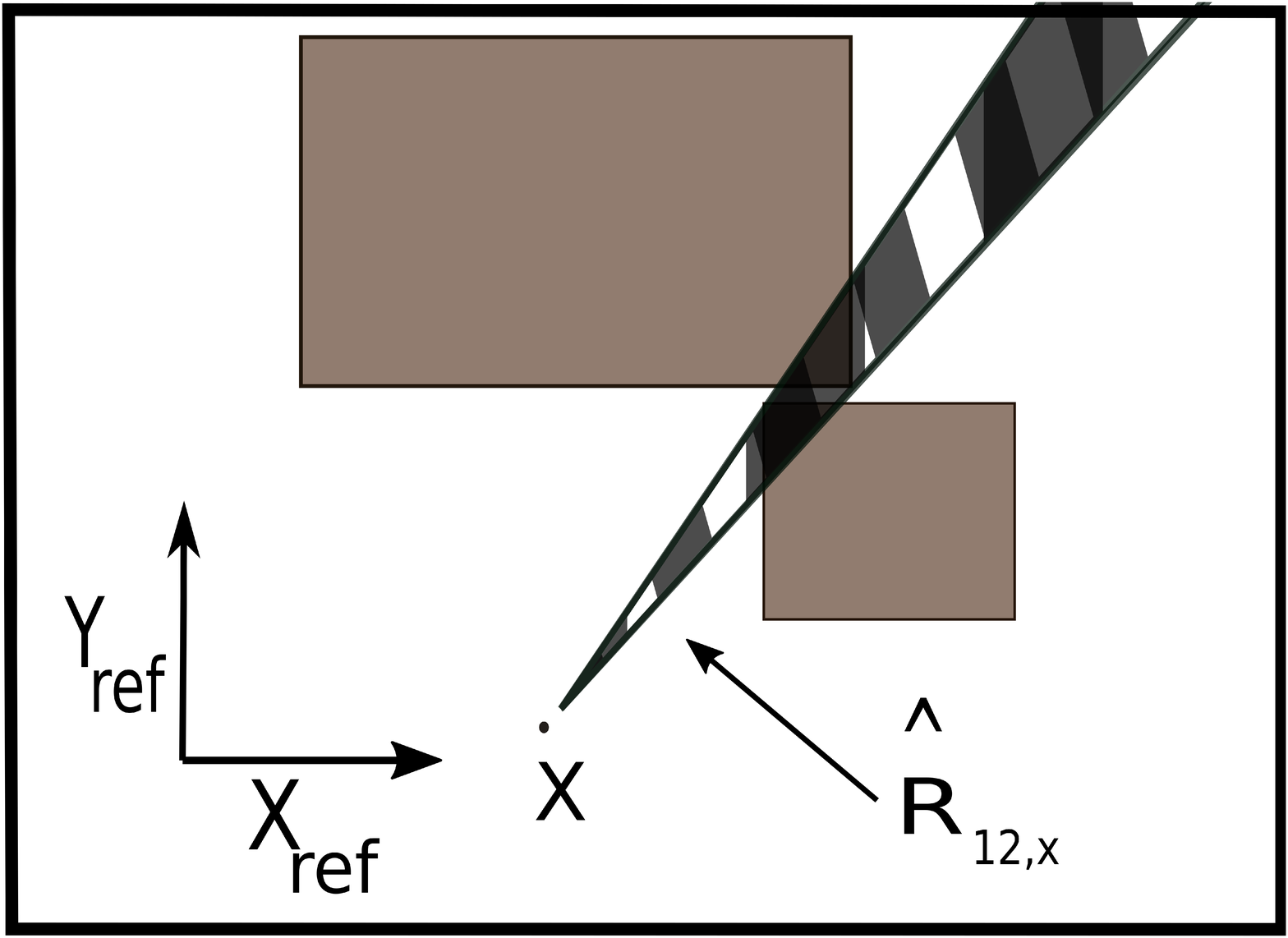}}
\caption{$\widetilde{\mathcal{R}_{12,\boldsymbol{X}}}$}
\label{fig:r12_tilde}
\end{subfigure}
\begin{subfigure}{0.23\textwidth}
{\includegraphics[scale=0.08]{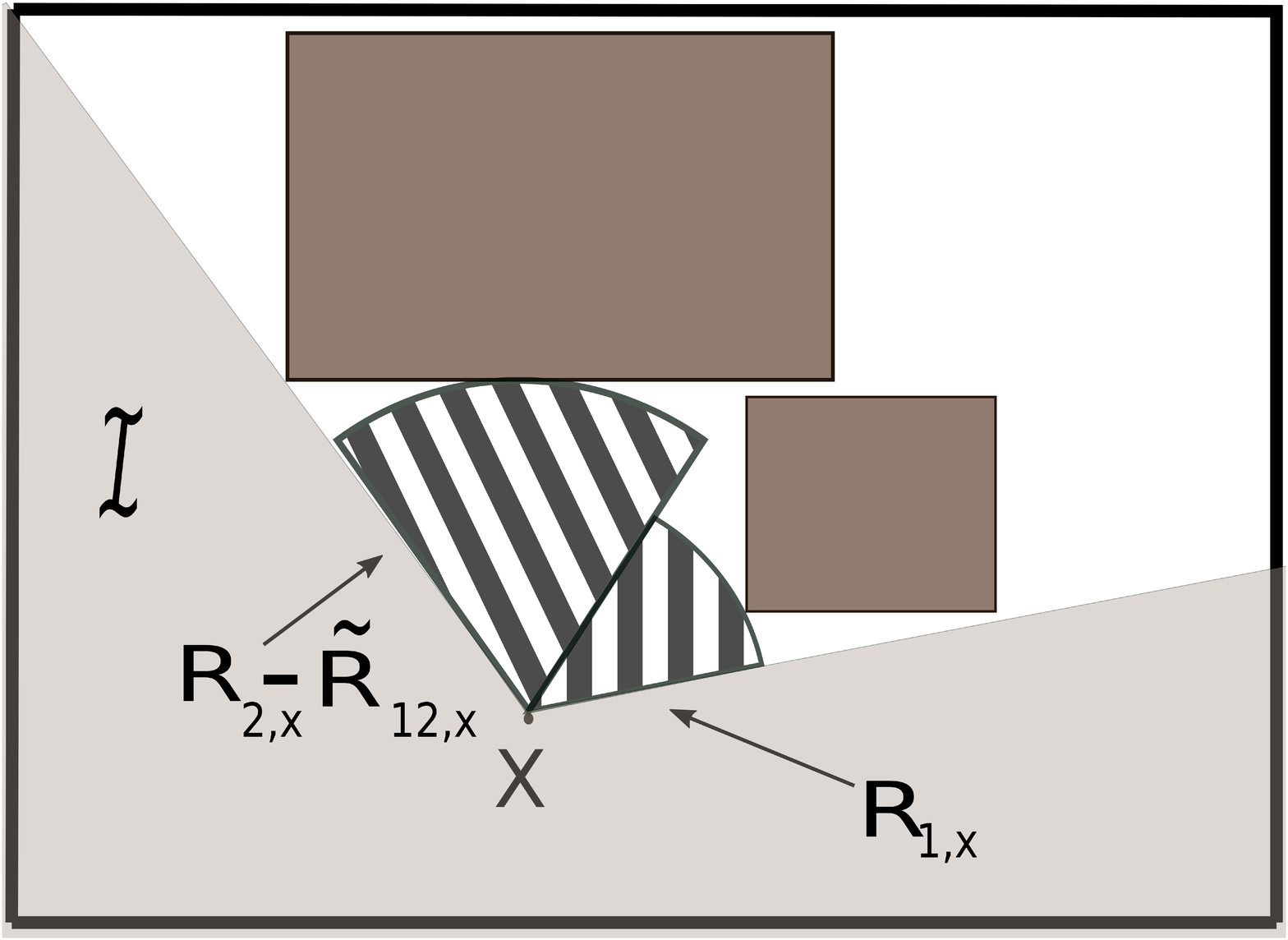}}
\caption{$\mathcal{S}_{\boldsymbol{X}}={\mathcal{R}_{1,\boldsymbol{X}}}\cup({\mathcal{R}_{1,\boldsymbol{X}}}-\widetilde{\mathcal{R}_{12,\boldsymbol{X}}})\cup\mathcal{I}$}
\label{fig:S_x}
\end{subfigure}
\begin{center}
\captionsetup[subfigure]{justification=centering}
\centering
\end{center}
\caption{An illustration of generalized shape about $\boldsymbol{X}$ in 2-D environment}
\label{fig:gse_2d_shape}
\end{figure*}

\subsubsection{Generalized shape generation}

The generalized shape $\mathcal{S}_{\boldsymbol{X}}$ generated about any point $\boldsymbol{X}$ is given by the $\texttt{Shape}$ function, which is detailed in Algorithm \ref{alg:shape_function}.
The function \texttt{Shape}$(\boldsymbol{P},\boldsymbol{X},\mathbb{X}_{\text{obs}})$  in Algorithm \ref{alg:shape_function} returns a value zero if a test point $\boldsymbol{P}$ is within the  generalized shape about the sampled point $\boldsymbol{X}$, and non-zero otherwise, that is,
\begin{align}
\mathcal{S}_{\boldsymbol{X}}({\boldsymbol{P}}) \left\{ \begin{array}{l}
=0;\quad\quad\text{if}\: {\boldsymbol{P}}\in\mathcal{S}_{\boldsymbol{X}}
\\
\neq 0,\quad\quad\text{otherwise}
\end{array}
\right.
\label{eq:P_in_ShapeX_condition}
\end{align}
Note that the $\texttt{sat}(.)$ function in Algorithm \ref{alg:shape_function} is given by  $\texttt{sat}({x})=1$ if $x\geq 0$ else $0$. 


\subsubsection{Steering $\boldsymbol{X}_\text{rand}$ to generate $\boldsymbol{X}_\text{new}$}($\texttt{Steer-GSE}(\boldsymbol{X_{\text{nearest}}},\boldsymbol{X_{\text{rand}}})$):

The randomly selected point $\boldsymbol{X}_\text{rand}$ is steered to a suitable point $\boldsymbol{X}_\text{new}$ as follows: 
\begin{itemize}
    \item If $\boldsymbol{X}_\text{rand}\in\mathcal{S}_{\boldsymbol{X}_\text{nearest}}$, then $\boldsymbol{X}_\text{new}$ is equal to $\boldsymbol{X}_\text{rand}$.
    \item 
    Otherwise, $\boldsymbol{X}_{\text{new}}$ is the point of intersection of the boundary of $\mathcal{S}_{\boldsymbol{X}_\text{nearest}}$ and the line joining $\boldsymbol{X}_{\text{nearest}}$ and $\boldsymbol{X}_{\text{rand}}$.
\end{itemize}
The point $\boldsymbol{X}_\text{new}$ thus computed is then added to the set $\mathbb{V}$. 

\subsubsection{Generation of new edges} 
In this step, first a set $\mathbb{X}_\text{intersect}$ is defined as $\mathbb{X}_\text{intersect}\triangleq \{\boldsymbol{X}\in\mathbb{V}|\mathcal{S}_{\boldsymbol{X}} \cap \mathcal{S}_{\boldsymbol{X}_\text{new}} \cap {\boldsymbol{X}\boldsymbol{X}_\text{new}} \neq \emptyset\}$.
Next, $\boldsymbol{X}_{\text{new}}$ is connected with the vertices $\boldsymbol{X}_\text{n} \in \mathbb{X}_{\text{near}}$ via newly generated edges $\boldsymbol{X}_\text{n} \boldsymbol{X}_\text{new}$, which are then added to the set $\mathbb{E}$. 

\subsubsection{Generating shortest path from $\boldsymbol{X}_\text{init}$ to $\boldsymbol{X}_\text{goal}$} Once a directed graph $\mathbb{G}=(\mathbb{V},\mathbb{E})$ connecting $\boldsymbol{X}_\text{init}$ and $\boldsymbol{X}_\text{goal}$ is found, using Dijkstra's algorithm, the minimum cost connected path is found from $\boldsymbol{X}_\text{init}$ to $\boldsymbol{X}_\text{goal}$.

\subsection{Probabilistic Completeness of the GSE Algorithm}\label{sec:probcomp}
We denote the vertex set and edge set of $\text{GSE}$ after $n$ iterations by $\mathbb{V}_{n}^{\text{GSE}}$ and $\mathbb{E}_{n}^{\text{GSE}}$, respectively.
 We indicate a summary of the  proof of  probabilistic completeness of the GSE algorithm. Further details are provided in \cite{probabilistic_completeness_gse}. We show this fact by establishing a few  properties of the GSE visibility function $f: 2^{\mathbb{X}_{\text{free}}} \rightarrow 2^{\mathbb{X}_{\text{free}}} $ defined as 
\begin{align}
   f(S) \triangleq \{\boldsymbol{X} \in \mathbb{X}_{\text {free }}\mid \mathcal{S}_{\boldsymbol{X}} \cap S \ne \emptyset\}
   \label{eq:visb_fn}
\end{align}
For the analysis of GSE, we study the iterates of this function $f^{(2)}(S)\triangleq f(f(S))$, $f^{(3)}(S)=f^{(2)} (f(S))$ etc. More precisely, we establish that
\begin{itemize}
    \item For any  path planning problem (conforming to the conditions laid down in our problem statement), there exists a finite constant $q>0$ such that $f^{(q)}(\boldsymbol{X}_{\text{goal}})=\mathbb{X}_{\text{free}}$.
    \item Let $\widehat{\mathbb{V}}_n^{\text{GSE}}$ denote the component of the GSE graph that contains the initial point $\boldsymbol{X}_{\text{init}}$. Then, the algorithm succeeds in finding a path from $\boldsymbol{X}_{\text{init}}$ to $\boldsymbol{X}_{\text{goal}}$ in $n$ iterations if $\widehat{\mathbb{V}}_n ^{\text{GSE}} \cap \mathbb{X}_{\text{free}} \neq \emptyset$.
    \item The probability that the vertex set of GSE fails to progress from the set $f^{(q)}(\boldsymbol{X}_{\text{goal}})$ to the set $f(\boldsymbol{X}_{\text{goal}})$ is bounded above by the tail probability of a Bernoulli random variable.
    \item Using Chernoff bounds, we show that the probability of failure of the GSE algorithm decays exponentially with the number of iterations of the algorithm.
\end{itemize}
Following this argument, we conclude that the GSE algorithm is probabilistically complete.

 \subsection{Analysis of Cost of Path Returned by the GSE Algorithm}\label{GSEasopt}
In this section, we analyse the cost of the paths returned by the GSE algorithm for investigating whether the GSE has asymptotic optimality property. In particular, we study performance of the algorithm in the  promenade problem, and conduct the analysis by defining a finite state machine called the Automaton of Sampling Diagrams (ASD) similar to the one defined in \cite{nechushtan2010sampling_diagram}. However, in this section, the automaton is modified to better suit the analysis of the GSE algorithm. 

\subsubsection{The promenade problem and the automaton of sampling diagrams for the GSE algorithm}

The promenade problem described below offers some perspective on how the GSE algorithm performs for certain problem types. The automaton we define  provides a way to study the critical events in the progress of the GSE algorithm. This helps determine the probability of achieving solutions of different qualities. For ease of understanding, here, the analysis is given for 2-D environments. However, it can easily be extended to more general settings of higher dimensional environments.

\subsubsection{Promenade problem}\label{sec:promnddesc}

The setting of this problem is as follows: Define $\mathbb{X}={[0,\alpha +2]}^2$ and $\mathbb{X}_{\text{obs}}=[1,\alpha +1]^2$. As usual we denote an arbitrary point by coordinates $(x,y)$. The initial and goal points are set to be on the left and right of the central obstacle. While the exact location need not be precise, we set $\boldsymbol{X}_\text{init}= (1-\epsilon,1+2\epsilon)$ and $\boldsymbol{X}_\text{goal}=(\alpha +1+ \epsilon,1+2\epsilon)$ where $\epsilon>0$ is much smaller than min$(\alpha,1)$. The aim is to generate a path between these two points. In this environment we define the following regions 
\begin{itemize}
\item $L_1=\{(x,y)\in \mathbb{X}| x+y\leq2 \}$
\item $L_2$ is the reflection of $L_1$ about the line $x= \alpha /2 +1$
\item $B_1=[0,1]\times [\alpha+1, \alpha +2]$
\item $B_2$ is the reflection of $B_1$ about the line $x= \alpha /2 +1$
\end{itemize}
A representative figure is given in Fig. \ref{fig:promenade}
We further classify all solutions to the promenade problem  into two types -
 \begin{itemize}
     \item Type-B solutions where the path returned by the algorithm  passes through $B_1$ and $B_2$, that is, passes through the region above the square obstacle.
     \item Type-L solutions where the path returned passes through $L_1$ and $L_2$,  that is, passes through the region  below the square obstacle.
 \end{itemize}
   It is immediately evident that Type-B solutions are of higher cost under the given conditions for the stated locations of the initial and goal points.
  
 \subsubsection {Automaton of sampling diagrams} \label{sec:defautomaton}
 
 \begin{figure}
     \centering
     \includegraphics[scale=0.2]{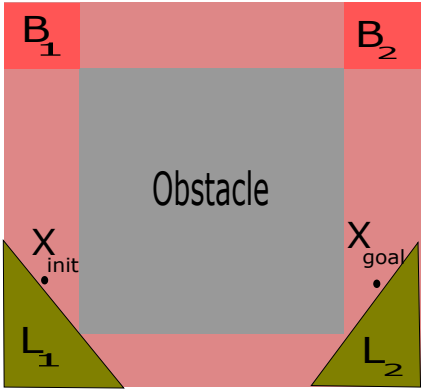}
     \caption{A schematic representation of the promenade problem}
     \label{fig:promenade}
 \end{figure}
 In the context of the problem above, we develop an automaton
to model the progress of the GSE algorithm. 
{Type-L and Type-B solutions are the two major types of solutions we wish to develop. Therefore, we consider two states $s_\text{accepting}$ and $s_\text{rejecting}$, associated with Type-L and Type-B solutions   respectively. These states are by necessity 'sink' states. Notice that the automaton begins with the state $s_{\text{init}}$. We are concerned primarily with the progress of the component of the GSE graph containing the initial point. Consequently, in order to better study the GSE graph, we define two interim states $s_1$ and $s_2$. Progress from $s_{\text{init}}$  to $s_{\text{accepting}}$ through the two interim states gives us a high cost solution, and we 'reject' all paths that might lead to low cost solutions via the rejecting state. Developing this heuristic, we have the following definition}
 \begin{definition} \label{def:asd}
    An automaton of sampling diagrams is a finite state machine with five states.The states are $s_{\text{init}}, s_{1} ,s_{2},s_{\text{accepting}}, s_{\text{rejecting}}
 $. The machine takes as input a point $\boldsymbol{X}$ sampled from the free-space. Based on the location of the point, decisions are made as to the progress of the automaton. Two sets $F_i$ and $R_i$ are affiliated with each state $s_{i}$, where $i \in \{1,2,\text{init},\text{rejecting},\text{accepting}\}$. The states are ordered $s_{\text{init}}\rightarrow s_{1}\rightarrow s_{2}\rightarrow s_{\text{accepting}}
 $. The state $s_\text{rejecting}$ is outside of this structure. The operations of the automaton are then described as below,
  \begin{itemize}
     \item If the automaton reaches a rejecting or an accepting state, then it does not progress further.
     \item Let the automaton be in a regular (neither accepting nor rejecting) state $s_i$. Suppose a point $\xi$ is sampled. If the  next vertex added to the graph generated by the GSE algorithm using the point $\xi$  lies in $R_i$ then the automaton moves to the rejecting state. 
     \item If the next vertex added to the graph lies in the region $F_i$, then the automaton progresses to the next state.
     \item If the vertex added to the graph of GSE is in neither $F_i$ nor in $R_i$, then the automaton stays in the same state.
 \end{itemize}
 \end{definition}

By suitably defining the regions $R_i$ and  $F_i$, we can classify the solutions returned by the GSE algorithm. The accepting state is associated with solutions that are of high cost. This motivates us to define these regions as:
 \begin{itemize}
     \item For all states $R_{i}=L_{1}$. 
     \item $F_{\text{init}}$ is a $\gamma \times \gamma$ square, with $\gamma$ a very small postive number. This square shares its  upper right corner with $B_1$ and is homothetic to it.
     \item $F_{1} \subset [1,\alpha+1]\times [\alpha +1,\alpha+2]$ is a sufficiently small  region that is parallel to $F_{\text{init}}$. We require that $F_2$ lies entirely within the shape of every point of $F_1$.
     \item $F_2$ is the mirror image of  $F_{\text{init}}$ about the line $x=\alpha/2 +1$
   \end{itemize}

A sample illustration of the regions $F_i$ (the forward regions) and $R_i$ is provided in in Fig. \ref{fig:asd}.

 \begin{figure}\label{fig:forwardregions}
     \centering
     \includegraphics[scale=0.2]{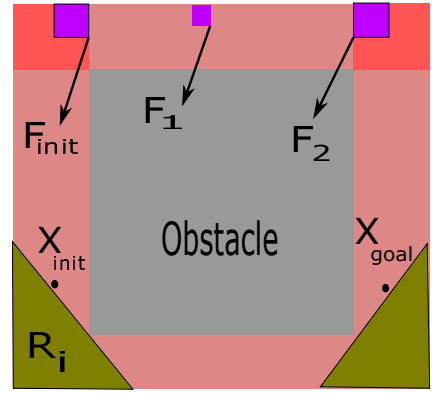}
     \caption{ An illustration of the regions $F_i$ and $R_i$}
     \label{fig:asd}
 \end{figure}
 
\subsubsection{Low-quality solutions in the GSE algorithm}
We analyse the solutions that the  GSE algorithm returns  for arbitrarily high number of iterations. This then allows us to conclude whether the GSE algorithm is asymptotically optimal or not. The reasoning for this is provided below. Note that $Y_n$ is defined as in Section \ref{defasyopt}

The proof below proceeds in two parts. Firstly, we prove that the automaton provides a schematic for the GSE algorithm in some sense. More precisely, it reaches certain states only when the GSE algorithm returns certain solutions. Paths returned by the GSE, passing through the regions $B_i$ are not optimal, we associate these paths with the accepting state of the automaton, thereby if we show that the probability that the automaton reaches the accepting state is bounded below, then we have the desired result. Before we prove this, we define a few quantities. 
\begin{definition}\label{def:swathregions}
   Associated with each (non-rejecting) state $s$ of the automaton, we define two regions ${H}^{+}(s)$ and ${H}^{-}(s)$ as follows.
   \begin{itemize}
    \item ${H}^{-}(s)=L_1$ for all non-rejecting states $s$
    \item$H^{+}(s_{\text{init}})=\emptyset$ 
    \item $H^{+}(s_{1})$ = $F_{\text{init}}$
    \item $H^{+}(s_{2})=F_1$
    \item $H^{+}(s_{\text{accepting}})=F_2$
    \end{itemize}
\end{definition}

We denote by $\mathcal{E}_{i}$ the swath of GSE, that is, the union of line segments that constitute its edges after $i$ iterations.

\begin{lemma}\label{theorem:automaton is good}
For all $\alpha\geq 2$, for any $\boldsymbol{X}_\text{init}$ in between $L_1$ and $B_1$ in the free space, and $\boldsymbol{X}_\text{goal}$ between $L_2$ and $B_2$, The automaton described in Section \ref{sec:defautomaton} moves to 
\begin{itemize}
    \item A rejecting state on all inputs(samples) if 
    GSE returns a Type-L solution
    \item  An accepting state only  if 
    GSE returns a Type-B solution
\end{itemize}       
\end{lemma} 
\begin{proof}
 Consider any two  arbitrary points $X_1,X_2  \in L_1^{c}$. Then, the line segment joining them does not pass through $L_1$ (with $L_1$ as in Fig. \ref{fig:promenade}) 
 Consequently, in order to have a Type-L path, we must have a vertex of GSE in $L_1$. But this moves the automaton to a rejecting state, by definition. The first part of the lemma is proved. Now, for the second part of the lemma, consider the following proposition.
\begin{prop}
Let a sequence of vertices $\xi_{1}\xi_{2}\dots\xi_{n}$ and swath $\mathcal{E}_n$ be generated by the GSE algorithm and read by the automaton. For each non-rejecting state  $s$, if $H^{+}(s)\neq \emptyset$, then  $\mathcal{E}_{n} \cap H^{+}(s) \neq \emptyset$.
\label{prop:non_rejecting_state_condition}
 \end{prop}
\begin{proof}
We prove this claim by induction on $n$. 
 
For $n=1$, if $\xi_1 \in F_{\text{init}}$ then the automaton progresses to state $s_1$. Since $H^{+}(s_{1})=F_{\text{init}}$, the claim is true.  

Next, we assume that the claim holds for some $n \in \mathbb{N}$ and the automaton is now in state $s_i$. Therefore, $H^{+}(s_i) \cap \mathcal{E}_n \neq \emptyset$. By statement of the proposition, $s_i$ is non-rejecting. The automaton reads the vertex $\xi_{n+1}$. In the event that the automaton does not progress, $H^{+}(s)$ remains fixed, referring to the induction hypothesis, the claim holds trivially. If the next added vertex lies in $F_i$, the automaton moves forward to the next state $s_{j}$, by Def. \ref{def:asd}. Since $s_j \neq s_{\text{init}}$, $H^{+}(s)$ is the same as $F_i$ (referring to  Def. \ref{def:swathregions}). we have $H^{+}(s_j) \cap \mathcal{E}_{n+1} \neq \emptyset $. This proves the claim in Proposition \ref{prop:non_rejecting_state_condition}.
\end{proof}
  
Lastly, we conclude that if the automaton reaches the accepting state on an input $\xi_{1}\xi_{2}\dots\xi_{n}$, then we can conclude that a vertex of the GSE graph lies in $F_2$. By suitably defining the region $F_2$, we ensure that $\boldsymbol{X}_{\text{goal}}$ (the goal point) is in the shape of points in $F_2$, and vice versa. Then the GSE algorithm will end on the $({n+1})^{\text{th}}$ iteration. Further, the swath of GSE will intersect $B_1$, $B_2$, and cannot intersect $L_1$. Consequently, the type of solution will be Type-B, a high cost sub-optimal solution. This proves the second part of lemma.
 \end{proof}

The next lemma shows that the Automaton of Sampling Diagrams achieves an accepting state with a probability that is bounded below, for an arbitrarily high number of iterations. 

\begin{lemma} \label{theorem:main asymptote}
If the sets $F_{i}$ in the definition of the automaton have a positive volume, then there exists a  constant $\pi_0>0$, and a natural number $N$ such that for all $n\geq N$, 
$\mathbb{P}(\text{Accept})\geq \pi_0$,
where, 'Accept' is the event that the automaton enters the accepting state after reading $n$ samples.
\end{lemma}

\begin{proof}
We begin by assuming the automaton is in a state $s_j$, where $j \notin \{\text{accepting},\text{rejecting}\}$. We define 
\begin{itemize}
    \item an event $\mathfrak{F}(s_j)$ as the event that the next sample is in a region which moves the automaton forward.
    \item an event $\mathfrak{R}(s_j)$ as the event that the  next sample is in a region which moves the automaton to a rejecting state.
\end{itemize}
We call these events as 'critical events'. The probability that the automaton moves forward (denoted $\pi(s_j)$), given the occurrence of a critical event can be given by 
\begin{equation}
    \pi(s_j) \triangleq \quad\mathbb{P}(\mathfrak{F}(s_j)|\mathfrak{F}(s_j) \cup  \mathfrak{R}(s_j)) = \frac{\mu(F_j)}{\mu(F_j)+\mu(R_j)} 
    \end{equation}
    
Here, $\mu$ is the Lebesgue measure.
{The above equation allows us to draw two conclusions. Firstly, given a sufficiently large input size, the automaton is guaranteed to enter either an accepting or a rejecting state. The set of states forms a finite state Markov Chain with $s_{\text{accepting}}$ and $s_{\text{rejecting}}$ as 'trap' states, and hence for arbitrarily large input sizes, the automaton is guaranteed to visit one of these. Consequently we can conclude that there exists a minimal $N$ such that for all $n>N$  after reading $\xi_{1},\xi_{2},\dots\xi_{n}$. 
\begin{equation}
    \mathbb{P}(\text{reached accepting or rejecting state})\geq 1-\delta
\end{equation}
Secondly, the probability of progressing from one state to the next is non-zero. Let the probability of progressing from state $s_j$ to the next be bounded below by some $p_j>\frac{\mu(+
F_j)}{\mu(\mathbb{X}_{\text{free}})}$, by virtue of the geometry of the problem. Denote $\pi_1 = \prod_{j} p_j$.}
 We then have that 
  \begin{equation}
  \quad\mathbb{P}(\text{Accept}) \geq (1-\delta)(\pi_1) = \pi_{0}
  \end{equation}
 Thus, the lemma follows.
\end{proof}
Thus, we have shown that the Automaton of Sampling Diagrams moves to an accepting state, with a probability that is bounded below, for an arbitrarily high number of iterations. Further by Lemma \ref{theorem:automaton is good}, we know that the automaton reaches an accepting state only if a Type-B solution is returned by the GSE algorithm. Hence the probability of returning a Type-B solution as defined in Section \ref{sec:promnddesc}
 is bounded below by a positive constant. This leads us to the following result that shows the GSE algorithm is not asymptotically optimal.
\begin{theorem}\label{thm:gseasyopt}
With $Y_n, c^{\star}$ defined as in Section \ref{defasyopt}, 
  \begin{equation} \label{equation: asymptotic fail}
     \quad\mathbb{P}(\{\underset{n\to\infty}{\limsup}\;{Y_n}= c^{\star}\})=0
 \end{equation} 
\end{theorem}
\begin{proof}

Since $Y_n>c^{\star}$ for all $n>N$ with positive probability, then $\underset{n\to\infty}{\limsup}\;{Y_n}>c^{\star}$ with positive probability, which then means 
\begin{equation}
    \mathbb{P}(\{\underset{n\to\infty}{\limsup}\;{Y_n}= c^{\star}\})<1
    \label{costeq}
\end{equation}

Probabilistic completeness property of the GSE algorithm guarantees that the algorithm returns a feasible path of finite cost. Therefore, invoking Lemma 25 of \cite{sertac_karaman2011sampling_rrt_star}, we have the following. 
\begin{lemma}{\cite{sertac_karaman2011sampling_rrt_star}} \label{lem:tailevent}
The event $\{{\limsup}\;{Y_n}= c^{\star}\}$ is  a tail event, and hence by Kolmogorov's zero-one law, it has probability either 0 or 1.
\end{lemma}

Lemma \ref{lem:tailevent} along with Eq. \eqref{costeq} verify that Eq. \eqref{equation: asymptotic fail} holds.
This proves the theorem.
\end{proof}

From the above, we conclude that while the GSE is probabilistically complete, it does not possess the  desired asymptotic optimality property.
\section{$\text{GSE}^{\star}$ Algorithm}\label{sec:GSE*}
\subsection{\label{GSE*alg}Revised Algorithm}

In this section, a modified version of the GSE algorithm is presented such that the resulting algorithm possesses asymptotic optimality property, while it also retains the advantages of the GSE algorithm. To this end, the $\text{GSE}^{\star}$ algorithm is formulated, which along with its detailed analysis (in Sections \ref{GSE*probcomp} and \ref{GSE*asopt}) forms one of the key contributions of this paper. The $\text{GSE}^{\star}$ in essence contains all the major steps used in the iterations of the GSE algorithm. In addition, a few more routines are added given in Line 10,13 and in Lines 17-23 in Algorithm \ref{algo:GSE*_algo}.

Before describing the revised algorithm, we introduce a few functions/notations used in the algorithm. First, the steer function ($\texttt{Steer}(\boldsymbol{X},\boldsymbol{Y},\eta)$) in Line 10 of Algorithm \ref{algo:GSE*_algo} returns a point $\boldsymbol{z}\in\mathbb{X}_{\text{free}} \cap \mathcal{B}_{\eta}(\boldsymbol{X})$, which is closest to the point $\boldsymbol{Y}$, that is, it returns a point lying in $\mathcal{B}_{\eta}(\boldsymbol{X})$ and on the line joining $\boldsymbol{X}$ and $\boldsymbol{Y}$. More precisely, $\texttt{Steer}(\boldsymbol{X},\boldsymbol{Y},\eta)$ is given as,
\begin{align}\label{steerfunc}
 \texttt{Steer}(\boldsymbol{X},\boldsymbol{Y},\eta)=\underset{z\in\mathcal{B}_{\eta}(\boldsymbol{X}) \cap \mathbb{X}_{\text{free}}}{\text{argmin}} \|\boldsymbol{z}-\boldsymbol{Y}\|  
\end{align}
Here, the parameter $\eta$ is arbitrary, and can be suited to the problem at hand.

Let $|\mathbb{V}|$ denote the number of vertices of the graph generated by the GSE algorithm. Following Lines 9 and 10 of Algorithm \ref{algo:GSE*_algo} at each iteration of the algorithm two vertices are added. Line 9 adds a vertex following the GSE algorithm, while Line 10 introduces another vertex following the steering procedure in Eq. \eqref{steerfunc}. 

We then define $r_n \triangleq \text{min}(\gamma_{\text{GSE}^\star} (\frac{\text{log}(|\mathbb{V}|)}{|\mathbb{V}|}))^{\frac{1}{d}},\eta)$. This choice of connection radius $r_n$ is motivated to ensure asymptotic optimality. In the proof of optimality, relevant bounds will be established for the constant $\gamma_{\text{GSE}^\star}$. The function  ($\texttt{Near}\left(\mathbb{V}, \boldsymbol{X}, r_n \right)$) returns a subset of the set of vertices of the graph generated by the GSE$^\star$ algorithm. These consist of those vertices which lie within the ball $\mathcal{B}_{r_n}{(\boldsymbol{X})}$. That is,
\begin{align}
   \texttt{Near}\left(\mathbb{V}, \boldsymbol{X}, r_n\right)=\mathbb{V}\cap \mathcal{B}_{r_n}(\boldsymbol{X})
\end{align}

Thus, the additional routines in the GSE$^\star$ algorithm over the earlier GSE algorithm are stated below. In Line 10 of Algorithm \ref{algo:GSE*_algo}, a new vertex is added using the steering function Eq. \eqref{steerfunc}. Subsequently, we use the $\texttt{Near}$ function, with connection radius as defined above, to establish a set of candidate vertices for connection in Line 18. Lines 19-22 use the $\texttt{Shape}$ function to check for viability of collision-free connection with each vertex and add edges to the edge set. Note that using the $\texttt{Shape}$ function is a stronger criterion than collision-checking, and is a key feature of the $\text{GSE}^\star$ algorithm similar to the GSE algorithm.

\subsection{Probabilistic Completeness of the $\text{GSE}^{\star}$ Algorithm}\label{GSE*probcomp}
We denote the vertex set and edge set of $\text{GSE}^{\star}$ after $n$ iterations by $\mathbb{V}_{n}^{\text{GSE}^\star}$ and $\mathbb{E}_{n}^{\text{GSE}^\star}$, respectively. The graph of the GSE$^\star$ algorithm is denoted $\mathbb{G}^{\text{GSE}^\star}_{n}$. As discussed in Section \ref{GSE*alg}, the $\text{GSE}^{\star}$ algorithm adds vertices and edges to the graph generated by the GSE algorithm. So, at any iteration, $\mathbb{V}_n ^{\text{GSE}}\subset \mathbb{V}_{n}^{\text{GSE}^\star}$ and $\mathbb{E}_n ^{\text{GSE}}\subset \mathbb{E}_{n}^{\text{GSE}^\star}$. We note that the GSE algorithm is probabilistically complete, as described in Section \ref{sec:probcomp} (details available in \cite{probabilistic_completeness_gse}). Since the graph generated by the GSE algorithm is a subgraph of the graph generated by the $\text{GSE}^{\star}$ algorithm, the $\text{GSE}^{\star}$ algorithm is probabilistically complete.

\begin{algorithm}[]
\caption{$\text{GSE}^{\star}$ Algorithm}
\begin{algorithmic}[1]
\State ${\mathbb{V}\gets \{\boldsymbol{X_{\text{init}}}, \boldsymbol{X_{\text{goal}}}\}},\;\;\mathbb{E}\gets \emptyset$
\State ${\mathcal{S}_{\text{init}}(\boldsymbol{P})\gets \texttt{Shape}(\boldsymbol{P},\boldsymbol{X_{\text{init}}},\mathbb{X}_{\text{obs}})}$
\State ${\mathcal{S}_{\text{goal}}(\boldsymbol{P})\gets \texttt{Shape}(\boldsymbol{P},\boldsymbol{X_{\text{goal}}},\mathbb{X}_{\text{obs}})}$
\For{$j=1\dots n$}
\State $\boldsymbol{X_{\text{rand}}}\gets \texttt{GenerateSample}$
\State $\mathcal{S}_{\text{rand}}(\boldsymbol{P})\gets \texttt{Shape}(\boldsymbol{P},\boldsymbol{X_{\text{rand}}},\mathbb{X}_{\text{obs}})$
\State $\boldsymbol{X_{\text{nearest}}}\gets \texttt{Nearest}(\mathbb{G}=(\mathbb{V},\mathbb{E}),\boldsymbol{X_{\text{rand}}})$
\State $\mathcal{S}_{\text{nearest}}(\boldsymbol{P})\gets \texttt{Shape}(\boldsymbol{P},\boldsymbol{X_{\text{nearest}}},\mathbb{X}_{\text{obs}})$
\State $\boldsymbol{X_{\text{new,g}}}\gets \texttt{Steer-GSE}(\boldsymbol{X_{\text{nearest}}},\boldsymbol{X_{\text{rand}}})$
\State $\boldsymbol{X_{\text{new}}}\gets \texttt{Steer}(\boldsymbol{X_{\text{nearest}}},\boldsymbol{X_{\text{rand}}},\eta)$
\State $\mathcal{S}_{\text{new,g}}(\boldsymbol{P})\gets \texttt{Shape}(\boldsymbol{P},\boldsymbol{X_{\text{new,g}}},\mathbb{X}_{\text{obs}})$
\State $\mathbb{X_{\text{intersect}}} \gets \texttt{IntersectShape}(G,\boldsymbol{X_{\text{new,g}}})$
\State $\mathbb{V} \gets \mathbb{V} \cup \{\boldsymbol{X_{\text{new,g}}},\boldsymbol{X_{\text{new}}}\}$
\For{$\boldsymbol{X_{\text{n}}}\in \mathbb{X}_{\text{intersect}}$}
\State $\mathbb{E}\gets \mathbb{E} \cup \{(\boldsymbol{X_{\text{n}}},\boldsymbol{X_{\text{new,g}}})\}$
\EndFor
\State $b\gets\gamma_{\text{GSE}^\star}(\log(|\mathbb{V}|) / |\mathbb{V}|)^{1/d}$
\State $\mathbb{U} \leftarrow \texttt{Near}\left(\mathbb{V}, \boldsymbol{X}_\text{new}, \text{min}(b,\eta)\right) $
\For{ $\boldsymbol{u} \in \mathbb{U}$}
\If{$\texttt{Shape}(\boldsymbol{u},\boldsymbol{X}_{\text{new}},\mathbb{X}_{\text{obs}})=0$}
\State $\mathbb{E}\gets \mathbb{E} \cup \{(\boldsymbol{u},\boldsymbol{X}_{\text{new}})\}$
\EndIf
\EndFor
\State $\mathbb{G}=(\mathbb{V},\mathbb{E})$
\State $Path_{\text{init,goal,shortest}} \gets\texttt{MinPath}(G,\boldsymbol{X_{\text{init}}},\boldsymbol{X}_{\text{goal}})$
\EndFor
\end{algorithmic}
\label{algo:GSE*_algo}
\end{algorithm}

\subsection{Asymptotic Optimality of the  $\text{GSE}^{\star}$ Algorithm}\label{GSE*asopt}
We now turn to proving that the GSE$^\star$ algorithm is asymptotically optimal. Since the GSE$^\star$ algorithm is quite similar in spirit to the RRG algorithm, we can reasonably expect this to be the case. In the following proof, we carefully follow the argument provided in \cite{sertac_karaman2011sampling_rrt_star} for the RRG algorithm. However, the selection of the parameter $\gamma_{\text{GSE}^\star}$ has been altered from $\gamma_{\mathrm{GSE}}$ to better suit for the analysis of the GSE$^\star$ algorithm. 

Let the optimal path between the initial and goal points be $\sigma^{\star}$. This path is assumed to have weak $\delta$-clearance. We assume that $\mu(\sigma^\star)=0$, where $\mu$ is the Lebesgue measure on $\mathbb{X}_{\text{free}}$.
The following lemma, while substantially similar to Lemma 50 of \cite{sertac_karaman2011sampling_rrt_star}, is included for completeness. Further, it serves to introduce several key parameters.
\begin{lemma}\label{lem:seqsigma}
For a path $\sigma^\star$ with optimal cost and weak $\delta$-clearance ($\delta>0$), there exists a sequence of paths $(\sigma_n)_{n=1}^{\infty}$ such that the following holds
\begin{itemize}
    \item Each path $\sigma_n$ has strong $\delta_n$-clearance, with $0<\delta_n \leq \delta$ and $\underset{n \rightarrow \infty}{\lim}\; \delta_n=0$
    \item $\underset{n \rightarrow \infty}{\lim}\; \sigma_n=\sigma^\star$
\end{itemize}
\end{lemma}
\begin{proof}

Let us define $\delta_{n}\triangleq\text{min}(\delta,\frac{(1+\phi)r_n}{2+\phi})$, for some bounded positive constant $\phi$. Clearly, the sequence $\delta_n$ (with $0\leq \delta_n \leq \delta$) is non-increasing and $\underset{n\to\infty}{\text{lim}}\delta_n =0$ (from the definitions of $\delta_n$ and $r_n$).
Consider the $n^{\text{th}}$ element of the sequence $(\mathcal{Y}_n)_{n=1}^{\infty}$ defined as 
$\mathcal{Y}_n=\{x \in \mathbb{X}_{\text{free}} \mid \mathcal{B}_{\delta_n}(x) \subseteq \mathbb{Y}_{\text{free}}\}$.

Now, we consider a homotopy $\psi$ where $\psi(0)=\sigma_{0}$ is a path with strong $\delta$-clearance. Further,  $\psi(1)=\sigma^{\star}$.
We can then define $\sigma_n=\psi(\alpha_n)$, where $\alpha_n=\underset{[0,1]}{\text{max}}\{t \mid \psi(t) \subseteq \mathcal{Y}_{n}\}$, that is, $\sigma_n$ is the path with strong $\delta_n$-clearance that is 'closest' to $\sigma^\star$. 
Since $\delta_n$ is a non-increasing sequence, by construction, $\mathcal{Y}_1\subseteq\mathcal{Y}_2\subseteq \dots \subseteq \mathbb{X}_{\text{free}}$. 
And, $\bigcup_{n=1}^{\infty}\mathcal{Y}_{n}=\mathbb{X}_{\text{free}}$. This establishes that the limit $\underset{n \rightarrow \infty}{\text{lim}} \alpha_{n}$ exists. Let this limit be $a$. Then, $\underset{n \rightarrow \infty}{\text{lim}} \psi(\alpha_{n})=\psi(a)$. Further, the clearance of $\psi(a)$ is given by $\underset{n \rightarrow \infty}{\text{lim}} \delta_{n}$. By the construction of the homotopy $\psi$, the clearance of $\psi(\alpha)$ is zero, only if $\alpha=1$. Since  $\underset{n \rightarrow \infty}{\text{lim}} \delta_{n}=0$, $a=1$. Thus, $\underset{n \rightarrow \infty}{\text{lim}} \alpha_{n}=1$. Consequently, $\underset{n \rightarrow \infty}{\text{lim}} \sigma_{n}=\sigma^{\star}$
\end{proof}

We have proved  that there exists a sequence of paths $(\sigma_n)_{n=1}^{n=\infty}$ such that each path $\sigma_n$ has strong $\delta_n$-clearance, and $\underset{n\to\infty}{\text{lim}}\delta_n =0$, and $\underset{n\to\infty}{\text{lim}}\sigma_n =\sigma^{\star}$. Next, we show that the probability that there exists a path  approximating $\sigma_n$ in $\mathbb{G}^{\text{GSE}^\star}_{n}$ tends to 1 as $n \rightarrow \infty$. This would ensure that the algorithm is asymptotically optimal.  
The following definition provides us with a framework toward generating the paths in GSE$^\star$ that approximate $\sigma_n$.
\begin{definition}\label{def:ballset}
We define the sets 
\begin{equation}\label{eq:ballset}
   \mathfrak{B}(\sigma_n, q_n,l_n)\triangleq\{B_{n,1},B_{n,2},B_{n,3},\dots B_{n,M_n}\} 
\end{equation}
Where each element $B_{n,m}$ is a ball of radius $q_n=\frac{\delta_n}{2+\phi}$, $l_n=\phi q_n$ and
\begin{itemize}
    \item The center of $B_{n,1}$ is $\sigma_n (0)=\sigma(0)$.
    \item The center of $B_{n,k}$ is $\sigma_n (b_k)$ , where $b_k \triangleq \min\{b\in [b_{k-1},1] \mid |\sigma(b)-\sigma(b_{k-1})|\geq \phi l_n\}$ for $k \in \{2,3,\dots M_n\}$
    \item The centre of the last ball must be the point $\sigma_n (1)$, that is, $b_{M_n}=1$.
\end{itemize}
\end{definition}

\begin{figure}
     \centering
     \includegraphics[scale=0.2]{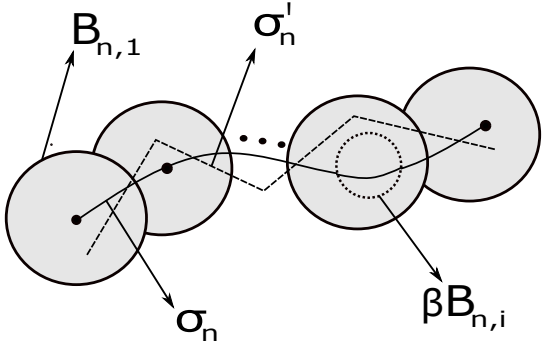}
     \caption{Ilustration of the paths $\sigma_n$ and $\sigma_n '$ along with the set of covering balls $ \mathfrak{B}(\sigma_n, q_n,l_n)$}
     \label{gse*balls}
 \end{figure}

A representation of the set $\mathfrak{B}(\sigma_n, q_n,l_n)$ is shown in Fig. \ref{gse*balls}.

We wish to show that each ball $B_{n,m}$ contains a vertex of GSE$^\star$ with higher probability as $n$ increases. This then allows us to consider paths in the graph of GSE$^\star$ that can approximate paths in the sequence $(\sigma_n)_{n=1}^{\infty}$. To prove this result, we first show that any point in the free-space is 'close' to a vertex of $\text{GSE}^{\star}$ after $n$ iterations with a probability that increases with $n$.

\begin{lemma}\label{lem:CMain}
Let $C_n$ at any iteration $n$ denote the event that for any $x \in \mathbb{X}_{\text{free}}$, there exists a vertex $v \in \mathbb{V}_n^{\text{GSE}^\star}$, such that $\|x-v\|\leq \eta$. Then $\mathbb{P}(C_n ^{c})\leq a e^{-b n}$ for some positive $\eta,a,b$.
\end{lemma}
\begin{proof}
Define the diameter of a set $\mathbb{S}$ by $d(\mathbb{S})$ for a set $\mathbb{S} \subset \mathbb{R}^2$ as follows:
\begin{align}
    d(\mathbb{S})=\underset{x,y \in \mathbb{S}}{\sup}\|x-y\|
\end{align}
We partition the set $\mathbb{X}_{\text{free}}$ into a collection of  sets $\{\mathcal{X}_i\}$ for $i \in \{1,2,\dots, M\}$, with $M$ finite, such that $d(\mathcal{X}_i)<\eta$ for all $i \in \{1,2,\dots, M\}$. An illustration is provided in Fig. \ref{figCmain}.
Define the indicator variable
\begin{equation}\label{eq:partitiongood}
    C_{n,i} \triangleq  
 \begin{cases}
              0&\text{if $\mathbb{V}_n ^{\text{GSE}^\star} \cap {\mathcal{X}_i}=\emptyset$}\\
              1&\text{otherwise}
 \end{cases}
\end{equation}
We can conclude that $C_n ^c \subseteq \bigcup_{i=1}^{M} C_{n,i}^c$. Consequently
 \begin{equation}\label{eq:c conclusion}
     \mathbb{P}(C_n ^c) \leq \sum_{i=1}^{M}\mathbb{P}(C_{n,i} ^c)
 \end{equation}
We now wish to show that $\mathbb{P}(C_{n,i} ^c)$ decays exponentially with increasing $n$.
 
\begin{figure}
     \centering
     \includegraphics[scale=0.2]{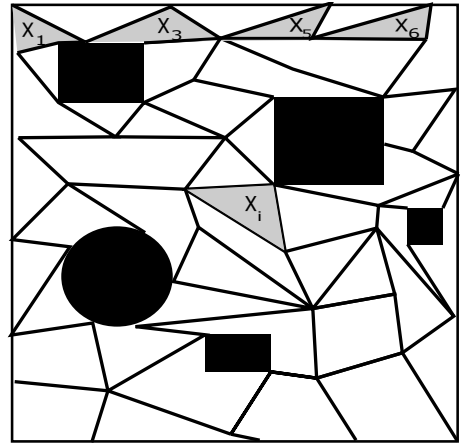}
     \caption{An illustration of  $\mathbb{X}_\text{free}$ and the corresponding exhaustive partitions $\mathcal{X}_1,\mathcal{X}_2,\dots,\mathcal{X}_M$. The black regions represent the obstacles.}
     \label{figCmain}
\end{figure}
 
To this end, define the function $g :2^{\mathbb{X}_{\text{free}}}\rightarrow 2^{\mathbb{X}_{\text{free}}}$ as
\begin{equation}
g(S)=\{y \in \mathbb{X}_{\text{free}} \mid B_{\eta}(x) \cap S \neq \emptyset\}
\end{equation}
The iterates of $g$ are denoted $g^{(2)}(S)=g(g(S))$, $g^{(3)}(S)=g^{(2)}(g(S))$ etc. For convenience, we denote $g(\{x\})$ (where $x \in \mathbb{X}_{\text{free}}$) by $g(x)$. Notice that $\mu(g^{(t)}(x))$ is an increasing function of $t$. We conclude that there exists an iterate of $g$ such that $\mathbb{X}_{\text{free}}\subseteq g^{(t)}(x)$, for any given point $x \in \mathbb{X}_{\text{free}}$.
Consider a point $x \in \mathbb{X}_{\text{free}}$. We wish to show that \begin{equation}
 \quad\mathbb{P}(g(x)\cap \mathbb{V}_n ^{\text{GSE}^\star}=\emptyset) \leq a_x e^{-b_x n}
\end{equation}
Let $i_k \triangleq \underset{i \in \mathbb N}{\text{min}}(\mathbb{V}_{k}\cap g^{(i)}(x) \neq \emptyset)$.
 Further 
 \begin{align}
 D_k \triangleq  
 \begin{cases}
              0&\text{if $i_{k+1}\geq i_k$}\\
              1&\text{otherwise}
 \end{cases}
 \end{align}
And let $D=\sum_{i=1}^{n} D_i$. 
Since passing to successive iterates of the function $g$ at least $t$ times, ensures that there will be a vertex of $\text{GSE}^{\star}$ in the set $g(x)$,

$\quad\mathbb{P}(\text{there is no vertex in $B_{\eta}(x)$})\leq \quad \mathbb{P}(D<t)$.

We claim that for $k^\text{th}$ iteration of GSE, $\mathbb{P}(D_k = 1)$ is bounded below by a positive fraction $p<1$ for all $k$ where $i_k > 1$.

At the $k^\text{th}$ iteration of GSE$^\star$, let there be $w$ $(\leq k)$ vertices of the GSE$^\star$ generated graph belonging to ${g}^{(i_k)}(x)$. We need to establish that for any location of these $w$ vertices in ${g}^{(i_k)}(x)$, there is a minimum probability of extending the graph to a point belonging to ${f}^{(i_k-1)}(x)$.

It is evident that the probability of $D_k=1$ increases as $w$ increases (assuming other $k-w$ vertices are kept fixed), because the number of candidate vertices for extension of the GSE$^\star$ generated graph increases. Hence, it suffices to consider the case, wherein $w=1$. Let this vertex be denoted by $\boldsymbol{Z}$. Let $\text{Vor}(\boldsymbol{Z})$ denote the Voronoi cell of the vertex $\boldsymbol{Z}$ in the Voronoi diagram consisting of the $k$ vertices of the GSE-generated graph so far. Because of the finite number $(k)$ of vertices in the GSE-generated graph, for the only vertex $\boldsymbol{Z}\in {g}^{(i_k)}(x)$, 
it follows that $\mu(\text{Vor}(\boldsymbol{Z}))$ is bounded below by some number, say $e_{i_k}$.
Now, for any arbitrary location of $\boldsymbol{Z}$ in the set ${g}^{(i_k)}(x)$, there is a subset of $\text{Vor}(\boldsymbol{Z})$, in which, when the new vertex is sampled, leads to extension of the GSE-generated graph to a point in ${g}^{(i_k -1)}(x)$. Let the minimum of this fraction over the set ${g}^{(i_k)}(x)$ be denoted as $g_{i_k}$. We can then conclude that for $p=\underset{i_k \leq m}{\text{min}}{\frac{{e_{i_k} g_{i_k}}}{{\mu(\mathbb{X}_{\text{free}})}}}$, the statement holds true.

Let us next consider a random variable $E=\sum_{i=1}^{t}E_i$, where the $E_i$ are i.i.d Bernoulli random variables ($\mathbb{P}(D_i =1)=p$). We know that $\mathbb{P}(D_i=1)>\mathbb{P}(D_i =1)=p$.
Consequently, the probability distribution of the random variable $C$ provides a worst-case (lower) bound on $\mathbb{P}(D>t)$, that is, $\mathbb{P}(D>t)>\mathbb{P}(E>t)$. Taking complement on both sides we have $\mathbb{P}(D<t)<\mathbb{P}(E<t)$. We can then use Chernoff bounds on the random variable $E$, so for some $k,l>0$:
\begin{align}
  \mathbb{P}(D<t) &\leq \mathbb{P}(C<t) \leq e^{-2t}e^{-ln +k/n}   
\end{align}
For sufficiently large $n$, $\mathbb{P}(D<t)\leq a_x e^{-b_x n}$. Since $\mathbb{P}(\mathbb{V}_{n} ^{\text{GSE}^\star}\cap g(x)=\emptyset)\leq  \mathbb{P}(D<t)$, 
we have $\mathbb{P}(\mathbb{V}_{n}\cap g(x)=\emptyset)\leq a_x e^{-b_x n}$.

Since we can select a suitable point in any given set $\mathcal{X}_i$ and apply the argument above, we can conclude that,therefore (from Eq. \eqref{eq:partitiongood}) $\mathbb{P}(C_{n,i}^c) \leq a_i e^{-b_i n}$ for some positive $a_i , b_i$ independent of $n$. Therefore, we can conclude that since the sum of finitely many exponentially decreasing functions is bounded above by an exponentially decreasing function, from Eq. \eqref{eq:c conclusion} one can conclude that $\mathbb{P}(C_n^{c})\leq a e^{-b n}$.
\end{proof}

Let $\rho \in (0,1)$ be a variable independent of $n$. The event $(\bigcap_{i=\Floor{\rho n}}^{n}C_i)$ indicates that between the $\Floor{\rho n}^{\text{th}}$ iteration and the $n^{\text{th}}$ iteration of the GSE$^\star$ algorithm, it is not necessary to steer any new sampled vertices using the $\texttt{Steer}$ function (Line 10 of Algorithm \ref{algo:GSE*_algo}), they are directly added to the graph. Since bounds on the probability of this event are very relevant to our proof, we have the following lemma.
\begin{lemma}\label{lem:Csub}
$\sum_{n=1}^{\infty}\mathbb{P}\left((\bigcap_{i=\Floor{\rho n}}^{n}C_i)^{c}\right)<\infty$
\end{lemma}
\begin{proof}
Since 
\begin{equation}
    \sum_{n=1}^{\infty}\mathbb{P}\left((\bigcap_{i=\Floor{\rho n}}^{n}C_i)^{c}\right)= \sum_{n=1}^{\infty}\mathbb{P}(\bigcup_{i=\Floor{\rho n}}^{n}C_i ^{c})\leq \sum_{n=1}^{\infty}\sum_{i=\Floor{\rho n}}^{n}a e^{-b n}
\end{equation}
We conclude that $ \sum_{n=1}^{\infty}\mathbb{P}((\bigcap_{i=\Floor{\rho n}}^{n}C_i)^{c})<\infty$ for all $a,b \in \mathbb{R}$.
\end{proof}
Lemmas \ref{lem:CMain} and \ref{lem:Csub} above lead us to the proof of Lemma \ref{approxlem} below.
\begin{lemma} \label{approxlem}
Define the event $A_{n,m}$ as the event that the ball $B_{n,m}$ contains a vertex of $\text{GSE}^{\star}$ after $n$ iterations of the $\text{GSE}^{\star}$ algorithm. Define $A_n\triangleq\bigcap_{m=1}^{M_n} A_{n,m}$, where $M_n$,$B_{n,m}$ etc. are as defined in Def. \ref{def:ballset}. Then, $\mathbb{P}(\underset{n \rightarrow \infty}{\liminf}A_n )=1$.
\end{lemma}
\begin{proof} 
We proceed by showing that $\sum_{n=1}^{\infty}\mathbb{P}({A_n ^{c}})$ is bounded. Then by the Borel-Cantelli lemma, we have $\mathbb{P}(\underset{n \rightarrow \infty}{\limsup}A_n ^{c})=0$. Thus, we can directly conclude that $\mathbb{P}(\underset{n \rightarrow \infty}{\liminf}A_n )=1$.

Let $s_n$ be the length of the path $\sigma_n$. We denote the volume of the unit ball in $\mathbb{R}^{d}$ by $\zeta_d$. Let $n_0=\underset{n \in \mathbb{N}}{\text{argmin}}(\delta_n < \delta )$. In what follows, we are concerned only with $n$ for which $n>n_0$. Since the distance between the centres of two successive balls (along the curve) in $\mathfrak{B}_{\sigma_n, q_n,l_n}$ is less than $l_n$, we can conclude that 
\begin{align}
M_n< \frac{(2+\phi)^2 s_n}{\phi(1+\phi) \delta_n}\;, \hspace{1cm}       \mu(B_{n,m})=\zeta_d q_n ^{d}
\end{align}
Let $\Floor{x}$ denote the largest integer less than a given real number $x$.
Observe that the event $\bigcap_{i=\Floor{\rho n}}^{n}C_i$, denotes the event that every vertex sampled between iterations $\Floor{\rho n}$ and $n$  is within $\eta$ distance of an existing vertex of GSE$^\star$. This ensures that every new vertex of the $\text{GSE}^{\star}$ graph does not need to be steered. Thus, we can treat the vertices in the set $\mathbb{V}^{\text{GSE}^\star}_n \setminus \mathbb{V}^{\text{GSE}^\star}_{\Floor{\rho n}}$ as being uniform random samples that are added directly to the graph generated by GSE$^\star$. Let $h \triangleq \frac{1+\phi}{(2+ \phi)^2}$.
The conditional probability of the event $A_{n,m}^{c}$ conditioned on the event $\bigcap_{i=\Floor{\rho n}}^{n}C_i$) can therefore be given by
\begin{align}
    \mathbb{P}\left(A_{n,m}^{c} \mid \bigcap_{i=\Floor{\rho n}}^{n}C_i\right) &\leq \left(1-\frac{\mu(B_{n,m})}{\mu(\mathbb{X}_{\text{free}})}\right)^{n-\Floor{\rho n}}\nonumber\\ 
    &\leq\left(1-\frac{\zeta_d r_n ^{d}}{h ^{d} \mu(\mathbb{X}_{\text{free}})}\right)^{(1-\rho)n}
    \label{eqn:A_nm_less_than}
\end{align}
We observe that $(1-x)^a \leq e^{-ax}$ for real $a$ and $x$. Thus the right hand side of Eq. \eqref{eqn:A_nm_less_than} is bounded above as follows:
 \begin{align}
 \left(1-({\zeta_d r_n ^{d}}/{h ^{d} \mu(\mathbb{X}_{\text{free}}))}\right)^{(1-\rho)n} &\leq e^{ \frac{(\rho -1) n\zeta_d r_n ^{d}}{h^{d} \mu(\mathbb{X}_{\text{free}})}}
 \label{eqn:one_minus_tau_d_less_than}
 \end{align}
Since $|\mathbb{V}^{GSE^\star}_n| \geq n$, we know that for sufficiently large $n$, $\frac{{\text{log}(\text{card}(\mathbb{V}))}}{\text{card}(\mathbb{V})} \leq \frac{\text{log}(n)}{{n}}$. Thus, we have the following inequality
\begin{align}
 e^{\left[{\frac{(\rho -1) n \zeta_d r_n ^{d}}{h^{d} \mu(\mathbb{X}_{\text{free}})}}\right]}
&\leq n^{\left[{\frac{-(1-\rho)(\gamma_{\text{GSE}^\star}^{d} \zeta_d)}{\mu(\mathbb{X}_{\text{free}})h^{d}}}\right]}
\end{align}
From Eqs. \eqref{eqn:A_nm_less_than} and \eqref{eqn:one_minus_tau_d_less_than} along with the definition of $A_n$,
\begin{align}\label{summability}
    \mathbb{P}\left(A_n ^{c} \mid \bigcap_{i=\Floor{\rho n}}^{n}C_i \right)&\leq M_n n^{\frac{-(1-\rho)(\gamma_{\text{GSE}^\star}^{d} \zeta_d)}{\mu(\mathbb{X}_{\text{free}})h ^d}}\nonumber\\
    &\leq \frac{(2+\phi)^2 s_nn^{\frac{1}{d}-\frac{(1-\rho)(\gamma_{\text{GSE}^\star}^{d} \zeta_d)}{\mu(\mathbb{X}_{\text{free}}) h ^d}} }{(1+\phi )\phi \gamma_{\text{GSE}^\star} {\text{log}(n)}^{\frac{1}{d}}}
\end{align}
In order to ensure that the left hand side of Eq. \eqref{summability} decreases with increasing $n$ and  to have the sequence on the  right hand side of Eq. \eqref{summability} summable, we require that the exponent of $n$ be less than -1, that is,  
\begin{align}
    \left[({(1-\rho)(\gamma_{\text{GSE}^\star}^{d} \zeta_d)}/{\mu(\mathbb{X}_{\text{free}})h ^d})-({1}/{d})\right] >1
\end{align}
If the above criterion is satisfied, the sequence $\sum_{n=1}^{\infty}\mathbb{P}(A_n ^{c} \mid \bigcap_{i=\Floor{\rho n}}^{n}C_i )$ is summable. Therefore, the condition on $\gamma_{\text{GSE}^\star}$ reduces to the following:
\begin{equation}
    \gamma_{\text{GSE}^\star} > h \left[\left(1+\frac{1}{d}\right)\left(\frac{\mu(\mathbb{X}_{\text{free}})}{1- \rho}\right)\right]^{\frac{1}{d}} 
\end{equation}
Lastly, from the definition of conditional probability, we have,
\begin{align}
    \mathbb{P}(A_n ^{c} \mid \bigcap_{i=\Floor{\rho n}}^{n}C_i ) &= \frac{\mathbb{P}\left(A_n ^{c} \cap (\bigcap_{i=\Floor{\rho n}}^{n}C_i) \right)}{\mathbb{P}(\bigcap_{i=\Floor{\rho n}}^{n}C_i)}\nonumber\\ &\geq \mathbb{P}\left(A_n ^{c} \cap \left(\bigcap_{i=\Floor{\rho n}}^{n}C_i\right) \right)\nonumber\\ & \geq \mathbb{P}(A_n ^{c}) - \mathbb{P}\left((\bigcap_{i=\Floor{\rho n}}^{n}C_i)^{c}\right) 
\end{align}
Thus, we have: 
   \begin{equation}
       \sum_{n=1}^{\infty}\mathbb{P}(A_n ^{c}) \leq \sum_{n=1}^{\infty} \mathbb{P}\left(A_n ^{c} \mid \bigcap_{i=\Floor{\rho n}}^{n}C_i \right) + \sum_{n=1}^{\infty} \mathbb{P}\left((\bigcap_{i=\Floor{\rho n}}^{n}C_i)^{c}\right)
   \end{equation}
      
Following Lemma \ref{lem:Csub} and Eq. \eqref{summability}, the right hand side  of the inequality above is summable. Therefore, $\sum_{n=1}^{\infty}\mathbb{P}(A_n ^{c})$ is bounded, and consequently by the Borel-Cantelli lemma, we have $\mathbb{P}(\underset{n \rightarrow \infty}{\limsup}A_n ^{c})=0$. Taking complement on both sides of $\mathbb{P}(\underset{n \rightarrow \infty}{\limsup}A_n ^{c})=0$ gives 
$\mathbb{P}(\underset{n \rightarrow \infty}{\liminf} A_n)=1$.
\end{proof}

\begin{obs}
\textnormal{ Observe that the distance between two points $x \in B_{n,i}$ and $y \in B_{n,i+1}$ can be bounded as }
 \begin{align}
  \|x-y\| \leq 2 q_n + l_n \leq \delta_n < r_n   
 \end{align}
\textnormal{Further, the distance between any point $x$ in a ball $B_{n,i}$ to the obstacle space is denoted $o_x$. Then $o_x \leq \delta_n - q_n \leq \delta_n$. Thus, for any point $x \in B_{n,i}$, every point in the adjacent balls lies in the shape of the point $x$.
We have shown in Lemma \ref{approxlem} that the event $A_n$ occurs infinitely often with probability one as $n$  increases to infinity. Since the distance between any two points in two consecutive balls is less than $r_n$, two vertices in consecutive balls are connected. Consequently, we can claim that  the $\text{GSE}^{\star}$ algorithm creates a path between the initial and goal points.}
\end{obs}
 
\begin{definition} \label{def:hppp} (Homogeneous Poisson Point Process (\cite{stoyan}))

Let $\text{Poisson}(a)$ be the Poisson random variable of intensity $a$. A homogeneous  Poisson point process of intensity $\nu$ on  $\mathbb{R}^{d}$ is a random set of countable points $P_{\nu}^{d} \subset \mathbb{R}^d$. The set $P_{\nu}^{d}$ is such that for disjoint and measurable sets  $S_1,S_2 \in \mathbb{R}^{d}$ , we have $|(P_{\nu}^{d}| \cap S_1 ) = \text{Poisson}(\mu(S_1)\nu)$ and $|(P_{\nu}^{d}| \cap S_2 ) = \text{Poisson}(\mu(S_2)\nu)$ with $\text{Poisson}(\mu(S_2)\nu)$ and $\text{Poisson}(\mu(S_1)\nu)$ independent.
\end{definition} 

 \begin{lemma}(\cite{stoyan}) \label{stoyan}
 Given a sequence  of points $(x_i)_{i=1}^{\infty}$ drawn independently and uniformly from a measurable set $L \subseteq \mathbb{R}^{d}$. The set $\{x_1,x_2,\dots,x_{\text{Poisson}(\tau)}\}$ 
 is the restriction to $L$ of a homogeneous Poisson point process of intensity $\frac{\tau}{\mu(L)}$
 \end{lemma}
 
 Let the set of paths in the graph generated by the $\text{GSE}^{\star}$ algorithm in $n$ iterations be $\mathfrak{P}_{n}$. Let us define $\sigma_{n}^{'}\triangleq \underset{\sigma^{'} \in \mathfrak{P}_{n}}{\argmin}\|\sigma_n-\sigma'\|_{\text{BV}}$. It now remains to be shown that the sequence of paths $\sigma_n ^{'}$ converges to the path $\sigma^{\star}$.

\begin{theorem}\label{gse*mainthm}
With the sequence $(\sigma'_n )_{n=1}^{\infty}$ defined as above, $\mathbb{P}({\underset{n \rightarrow \infty}{\lim}\|\sigma_n -\sigma'_n \|_{\text{BV}} =0})=1$
\end{theorem}
\begin{proof}
We establish this by showing that $\sum_{n=1}^{\infty}\mathbb{P}(\|\sigma_n -\sigma'_n \|_{\text{BV}}>\epsilon)$ is finite for any $\epsilon>0$. Thus, by a straightforward application of the Borel-Cantelli lemma, we can conclude that $\mathbb{P}({\underset{n \rightarrow \infty}{\lim}\|\sigma_n -\sigma_n' \|=0})=1$.
\begin{figure*}[]
\captionsetup[subfigure]{justification=centering}
\centering
\begin{subfigure}{0.245\textwidth}
{\includegraphics[scale=0.245]{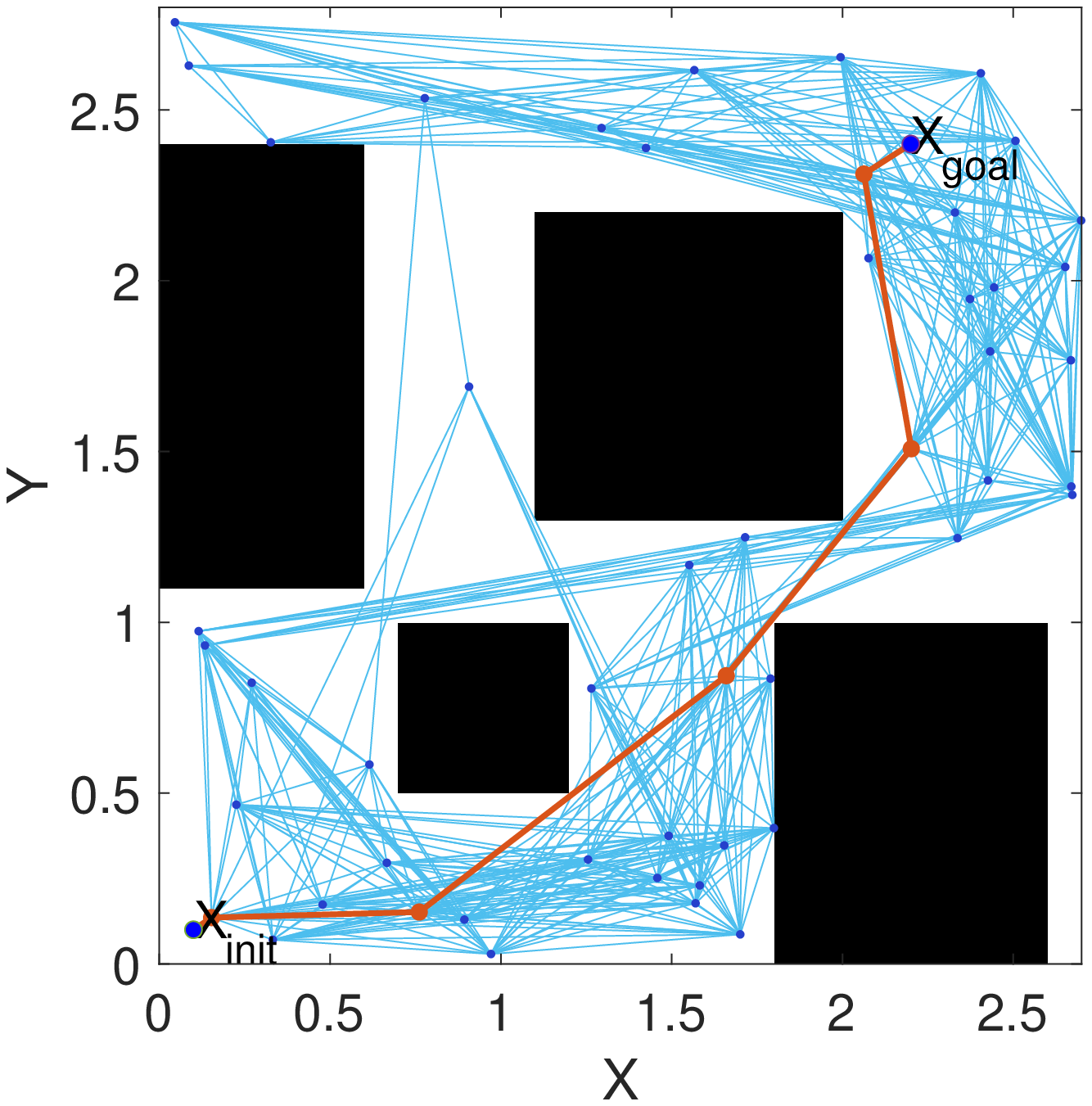}}
\caption{GSE for $m=4$. Cost: 3.64 }
\label{fig:iterations_50_gse}
\end{subfigure}
\begin{subfigure}{0.23\textwidth}
{\includegraphics[scale=0.245]{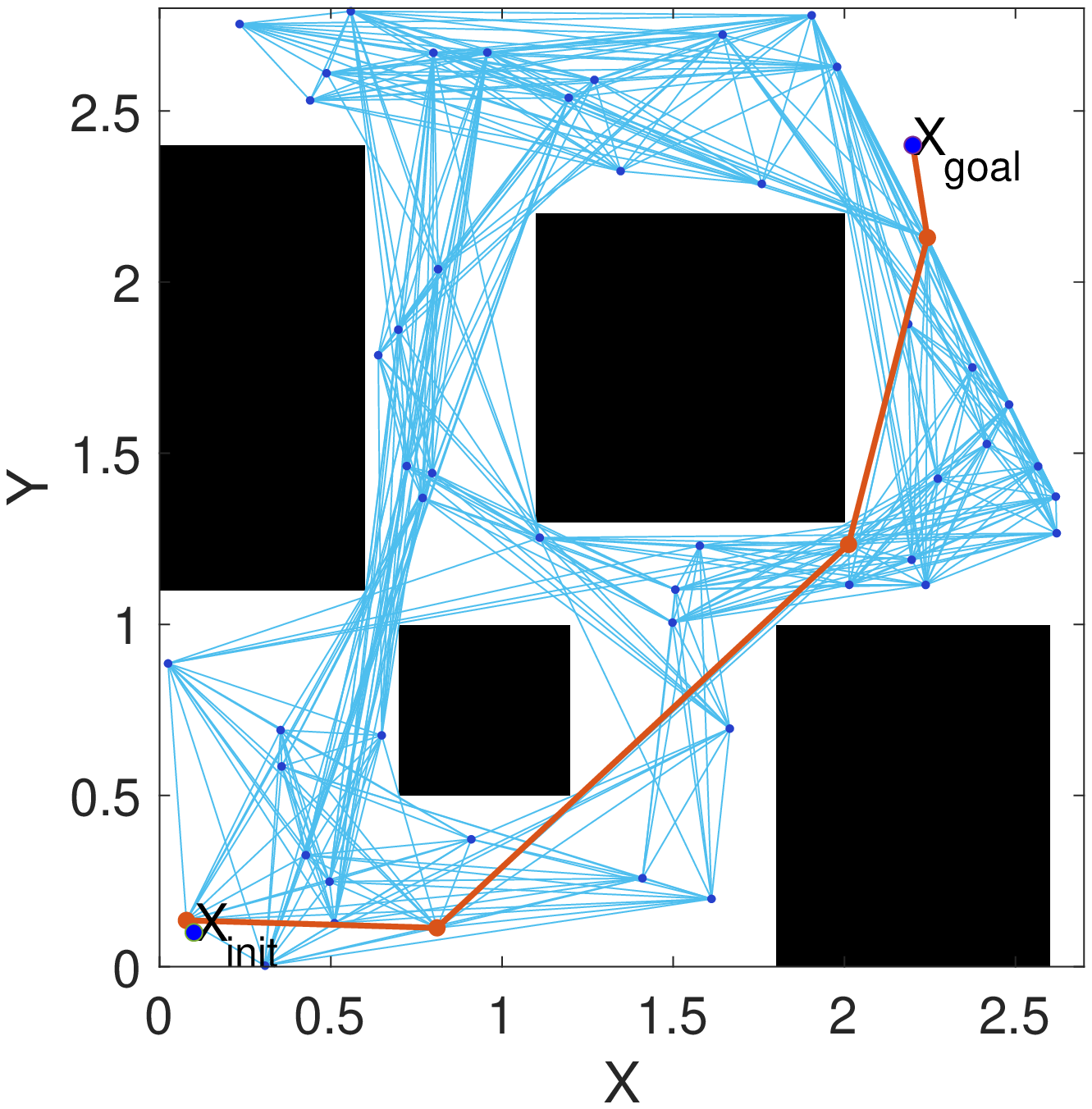}}
\caption{GSE$^\star$ for $m=4$. Cost: 3.61}
\label{fig:iterations_50_gse_star}
\end{subfigure}
\begin{subfigure}{0.245\textwidth}
{\includegraphics[scale=0.245]{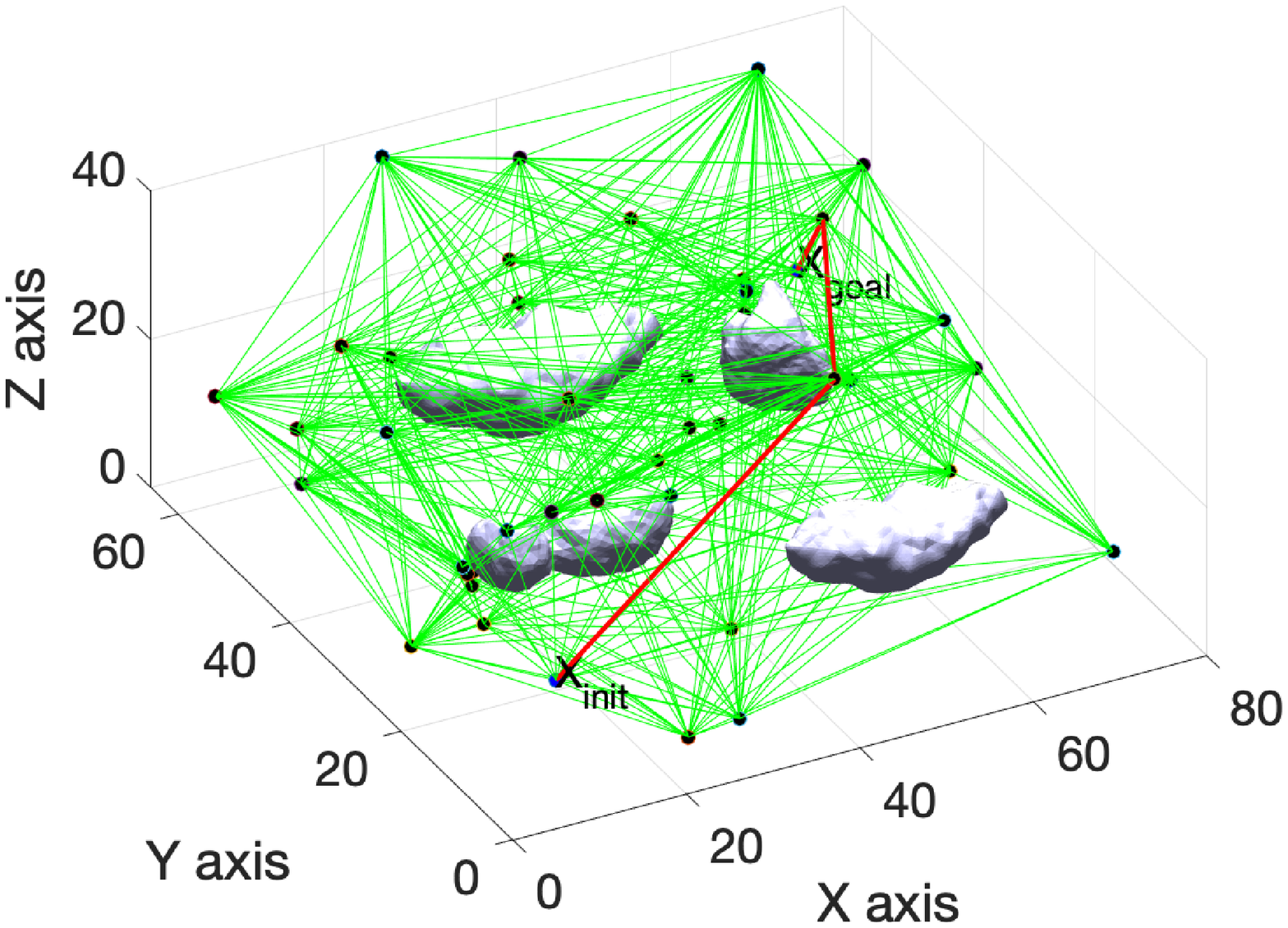}}
\caption{GSE for $m=4$. Cost: 84.32}
\label{fig:gse_3d}
\end{subfigure}
\begin{subfigure}{0.245\textwidth}
{\includegraphics[scale=0.245]{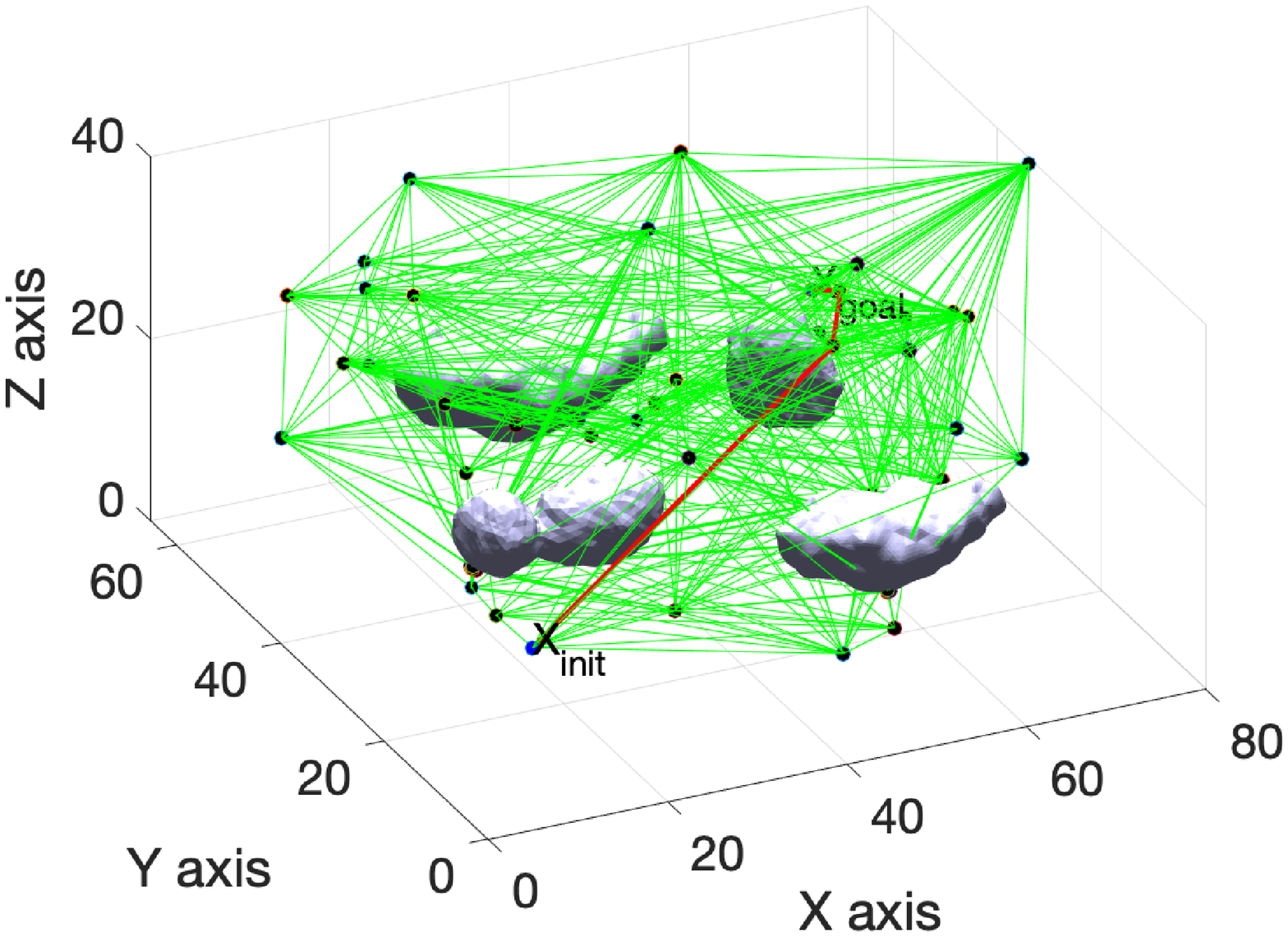}}
\caption{GSE$^\star$ for $m=4$. Cost: 81.37}
\label{fig:gse_star_3d}
\end{subfigure}
\begin{center}
\captionsetup[subfigure]{justification=centering}
\centering
\end{center}
\caption{Illustration of the shortest feasible paths (red lines) found by the GSE and GSE$^\star$ algorithms in 50 iterations, in 2-D and 3-D environments with 4 obstacles.}
\label{}
\end{figure*}
\begin{figure}[]
\captionsetup[subfigure]{justification=centering}
\centering
\begin{subfigure}{0.23\textwidth}
{\includegraphics[scale=0.23]{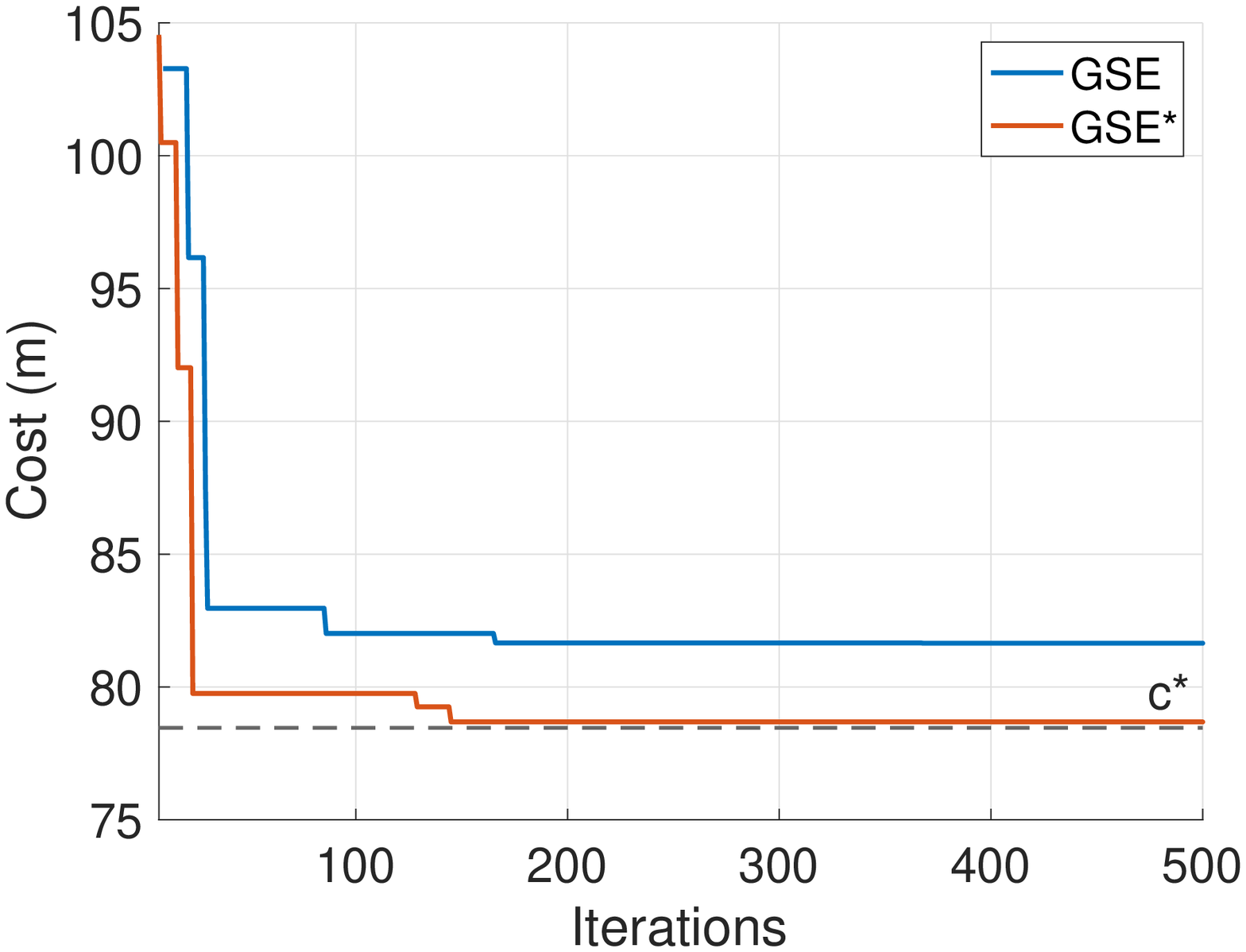}}
\caption{Convergence plot for $m=4$}
\label{fig:convergence_plot_4}
\end{subfigure}
\begin{subfigure}{0.23\textwidth}
{\includegraphics[scale=0.23]{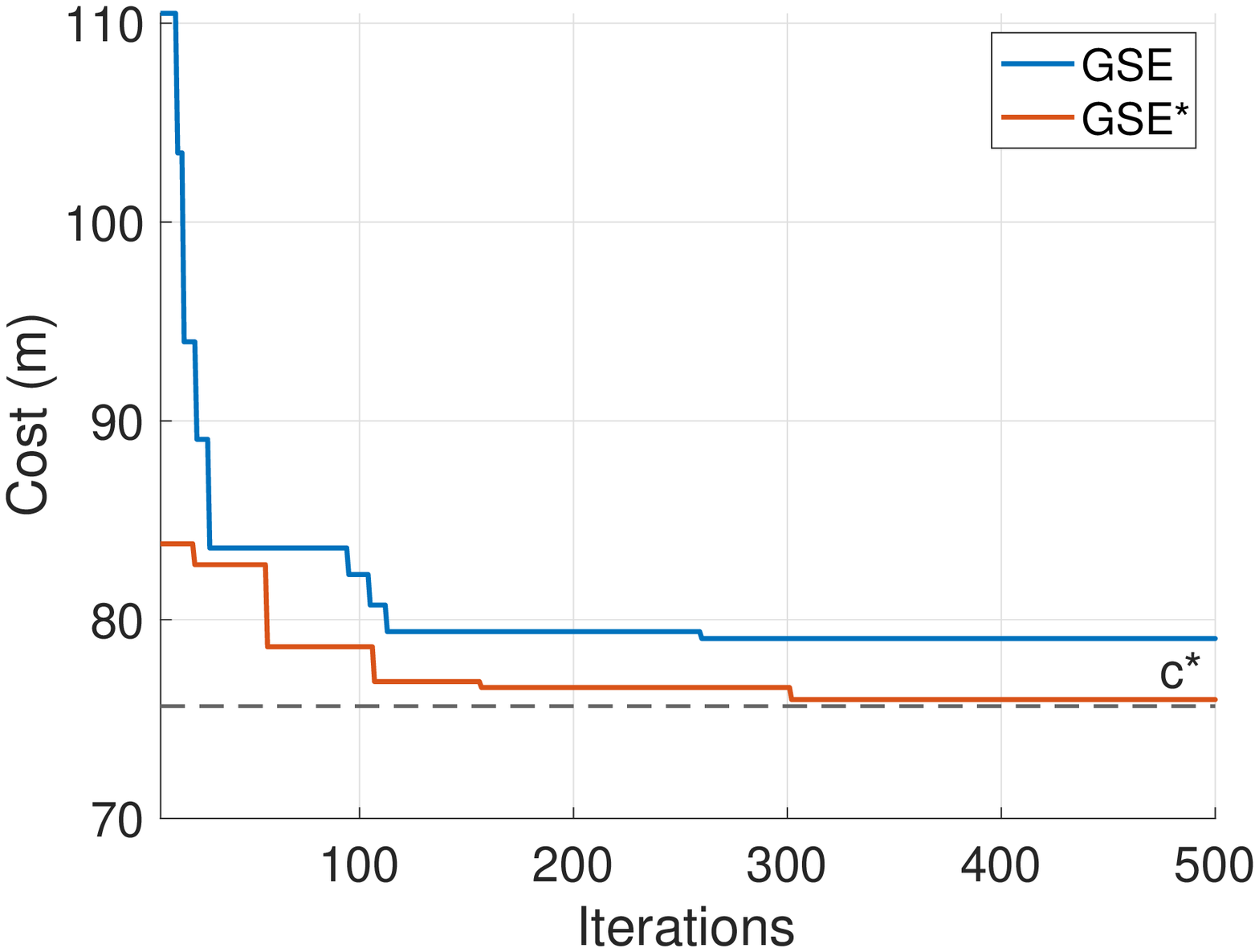}}
\caption{Convergence plot for $m=16$}
\label{fig:convergence_plot_16}
\end{subfigure}
\begin{center}
\captionsetup[subfigure]{justification=centering}
\centering
\end{center}
\caption{Convergence to the optimal cost for the GSE and GSE* algorithms}
\label{}
\end{figure}
We fix $\epsilon>0$, Let $\alpha , \beta \in (0,1)$ be two constants independent of $n$. Define $\beta B_{n,m}$ to be a ball concentric with  $B_{n,m}$, and has a radius of $\beta q_n$. Fig. \ref{gse*balls} provides  a representation of the covering balls and paths.

We define the indicator variables $I_{n,m}$ as follows,
\[
I_{n,m}\triangleq
\begin{cases}
           0 &\text{if there is a vertex of GSE* in the ball $\beta B_{n,m}$ }\\
           1 &\text{otherwise}
\end{cases}
\]
We define $K_n=\sum_{m=1}^{M_{n}}I_{n,m}$. This measures the number of balls $\beta B_{n,m}$ which do not contain a vertex from the graph constructed by the $\text{GSE}^{\star}$ algorithm. Let $L$ be an upper bound on the lengths of the paths $\sigma_n$. We can conclude that if $K_n \leq \alpha M_n$, then $\|\sigma_n -\sigma'_n\|_{\text{BV}}\leq (\sqrt{2}\alpha + \beta (1-\alpha))L$. Hence,
\begin{equation}
    \{K_n\leq \alpha M_n\} \subseteq \{\|\sigma_n -\sigma'_n\|_{\text{BV}}\leq \sqrt{2} L(\alpha +\beta)\}
\end{equation}
This allows us to conclude that,
 \begin{equation} \label{sigma' and k_n}
    \mathbb{P}(\{K_n \geq \alpha M_n\}) \geq  \mathbb{P}(\{\|\sigma_n -\sigma'_n\|_{\text{BV}}\geq \sqrt{2} L(\alpha+\beta)\})
 \end{equation}
Thus, to prove the theorem, we need to show that  $\sum_{n=1}^{\infty}\mathbb{P}(\{K_n \geq \alpha M_n\}) < \infty$.  
In order to show this, we first poissonize the sampling process, to obtain a homogeneous Poisson point process as in Def. \ref{def:hppp} and Lemma \ref{stoyan}. Subsequently we de-poissonize the process to obtain the relevant bounds on $\mathbb{P}(\{K_n \geq \alpha M_n\})$.

Consider $\lambda <1$, independent of $n$. Consider the random variable $\text{Poisson}(\lambda n)$. Let the points sampled by $\text{GSE}^{\star}$ (independently and uniformly at random) be $\{x_1,x_2,\dots,x_n\}$. By Lemma \ref{stoyan}, we can conclude that the set $P_{\lambda n }^{d}=\{x_{1},x_{2},\dots,x_{{\text{Poisson}(\lambda n)}}\}$ is a homogeneous Poisson point process of intensity $\frac{\lambda n}{\mu(\mathbb{X}_{\text{free}})}$. 
Let us denote by $\widehat{K}_n$ to be the number of balls in $\beta \mathfrak{B}_{n,q_n,l_n}={\beta B_{n,1}, \beta B_{n,2} \dots \beta B_{n,M_n}}$ that fail to have a point from the set $P_{\lambda n}^{d}$ within radius $\beta q_n$ of their centres. Notice that the underlying process is different from the process for the  original variable $K_n$. Relating the original process and the poissonized version (following \cite{penrose2003random}),

\begin{align}
    \mathbb{P}(K_n \geq \alpha M_n) \leq  \mathbb{P}(\widehat{K}_n \geq \alpha M_n) + \mathbb{P}(\text{Poisson}(\lambda n) \geq n)
\end{align}

 We can directly offer the bound 
 $\mathbb{P}(\{\text{Poisson}(\lambda n) \geq n\}) \leq e^{-c n}$ where $c>0$ is a constant, by the definition of a Poisson random variable. In order to bound $\mathbb{P}(\widehat{K}_n \geq \alpha M_n)$, observe that for sufficiently small $\beta$, the sets $\beta B_{n,m}$ are disjoint. Thus, the events that each of these balls do not contain a point from $P_{\lambda n}^{d}$ (which is a homogeneous Poisson point process) are mutually independent. 
 The probability $p_n$ that any given set does not contain a ball is therefore given as follows
\begin{align}
  p_n = e^{- \frac{\lambda n}{\mu(\mathbb{X}_{\text{free}})}}  
\end{align}
The variable  $\widehat{K}_n$ can therefore be compared to a binomial random variable (with parameters $M_n$ and $p_n$) as  follows
\begin{align}
    \mathbb{P}(\widehat{K}_n \geq \alpha M_n) \leq  \mathbb{P}(\text{Binomial}(M_n,p_n) \geq \alpha M_n) \leq e^{- M_n p_n}
    \label{eqn:k_n_less_than_e}
\end{align}
Therefore from Eq. \eqref{eqn:k_n_less_than_e}, we can conclude the following
\begin{equation}
    \sum_{n=1}^{\infty} \mathbb{P}(K_n \geq \alpha M_n) \leq \sum_{n=1}^{\infty}(e^{-cn} + e^{-M_n p_n}) \leq \infty
\end{equation}
Notice that the constants $\alpha, \beta $ are arbitrary. Consequently we can set $\epsilon=\sqrt{2}(\alpha + \beta)L )$, for arbitrary $\epsilon>0$ From the inequality \eqref{sigma' and k_n} and the equation above we know that 
\begin{equation}
     \sum_{n=1}^{\infty} \mathbb{P}(\{\|\sigma_n -\sigma'_n\|_{\text{BV}}> \epsilon \}) < \infty
\end{equation}
By the Borel-Cantelli lemma and the equation above, we have $\mathbb{P}(\{\|\sigma_n -\sigma'_n\|_{\text{BV}}=0)=1$.

Since $\underset{n\rightarrow\infty}{\lim}\sigma_n =\sigma^\star$ from Lemma \ref{lem:seqsigma}, by application of the triangle inequality,
\begin{align}
  \mathbb{P}(\{\underset{n \rightarrow \infty }{\lim}c(\sigma'_{n})= c(\sigma^\star)\})=1  
\end{align}
Thus, the asymptotic optimality of the GSE$^\star$ is proved.
\end{proof}
 
 \begin{obs}
\textnormal{  As a corollary of Theorem \ref{gse*mainthm}, we notice that $\|\sigma'_n - \sigma^\star\|_{\text{BV}} \leq \|\sigma'_n - \sigma_n\|_{\text{BV}} +\| \sigma_n - \sigma^\star\|_{\text{BV}}$, and hence the $\mathbb{P}(\|\sigma'_n - \sigma^\star\|_{\text{BV}}> \epsilon)$ decays with  increasing $n$ for all $\epsilon >0$, exponentially.}
 \end{obs}

\section{Simulation Results}\label{sec:results}
In this section, we present numerical validation of the properties of the algorithms GSE and GSE$^\star$ analyzed in the paper. Simulation studies have been carried out in 2-D and 3-D environments having 4 and 16 obstacles representing low to high obstacle densities using MATLAB R2020a on Intel Core i7 2.2GHz processor. Simulations and comparison studies related to probabilistic completeness of the GSE have already been provided in \cite{probabilistic_completeness_gse} and hence, are omitted here for brevity. First, feasible shortest paths generated using the GSE and GSE$^\star$ algorithms over 50 iterations in the 2-D and 3-D environments each having both 4 and 16 obstacles are illustrated in Figs. \ref{fig:iterations_50_gse}, \ref{fig:iterations_50_gse_star} and Figs. \ref{fig:gse_3d}, \ref{fig:gse_star_3d}, respectively. 

Throughout the simulations, the cost function is taken to be the Euclidean path length. Since PRM$^\star$ is asymptotically optimal, the approximate optimal cost ($c^\star$) is found using the PRM$^\star$ which is run over 8000 iterations. Figs. \ref{fig:convergence_plot_4} and \ref{fig:convergence_plot_16} depict the convergence of costs to the optimal one for 4 and 16 obstacles in a 3-D workspace. For a given number of obstacles, four different randomly generated workspaces are considered. For each workspace, 100 simulations, of 500 iterations in each simulation, have been run for each algorithm. The cost obtained is averaged over each set of 100 simulations for a given workspace. This average cost is again averaged over the four different workspaces, thus providing the costs depicted in the convergence plots in Figs. \ref{fig:convergence_plot_4} and \ref{fig:convergence_plot_16}.

The numerical simulations are primarily aimed at demonstrating the performance of the GSE$^\star$ algorithm in comparison with the original version  of the GSE algorithm in terms of the convergence of cost function to the optimal cost. It is observed that the GSE$^\star$ converges to the optimal cost in around 150 and 300 iteration for 4 and 16 obstacles, respectively, whereas the GSE is found not to converge even in 500 iterations to the optimal cost. This justifies the claim of asymptotic non-optimality of the GSE in Section \ref{GSEasopt} and the asymptotic optimality of the GSE$^\star$ algorithm in Section \ref{GSE*asopt}.

\section{\label{sec:conclusion}Conclusion}
The recently-proposed Generalized Shape Expansion (GSE) algorithm was shown to be probabilistically complete. In this paper, we have extended the study on the GSE to show that it is not asymptotically optimal by observing its non-zero probability of generating low-quality solutions to the promenade problem. In order to improve in the context of asymptotic optimality, a modified GSE algorithm, namely the $\text{GSE}^{\star}$ algorithm, has been presented. A detailed mathematical analysis has been given to prove both probabilistic completeness and asymptotic optimality of the GSE$^\star$. Further, simulation studies comparing the performance of these algorithms in various environments have shown marked improvement in asymptotic convergence to the optimal cost by the GSE$^\star$ over the GSE, which justifies the presented theoretical results. The key feature of the generalized shape expansion-based algorithms is the use of the generalized shape to generate collision-free paths. Leveraging these shapes have been found to greatly reduce the number of iterations required to obtain a feasible path attaining a minimally connected graph. While maintaining the notion of the GSE$^\star$ to gain computational advantages could be a potential future scope of research.







\ifCLASSOPTIONcaptionsoff
  \newpage
\fi

\bibliography{main.bib}
\end{document}